\documentclass[a4,11pt]{article}
\usepackage{xcolor}
\usepackage{float}
\usepackage{amssymb}
\usepackage{amstext}
\usepackage{amsmath}
\usepackage{amscd}
\usepackage{amsthm}
\usepackage{amsfonts}
\usepackage{enumerate}
\usepackage{graphicx}
\usepackage{latexsym}
\usepackage{mathrsfs}
\usepackage{mathtools}
\usepackage{cases}
\usepackage{verbatim}
\usepackage{bm}
\usepackage{tikz}
\usepackage{subcaption}
\usepackage{tabularx}
\usepackage{caption}
\usepackage{adjustbox}
\usepackage{authblk}
\usepackage{hyperref} 
\usepackage{cleveref}
\usetikzlibrary{matrix}
\usepackage{comment}
\usepackage{tikz-cd}
\usepackage{pb-diagram}
\usepackage[letterpaper,top=2cm,bottom=2cm,left=3cm,right=3cm,marginparwidth=2cm]{geometry}
\usepackage[subrefformat=parens]{subcaption}
\usepackage{wrapfig}
\usepackage{booktabs}

\newcommand{ \eqdef }{   \ensuremath{\stackrel{\mbox{\upshape\tiny def.}}{=}}
}

\newcommand{\R}{\mathbb{R}}
\newcommand{\relu}{\operatorname{ReLU}}

\newcommand{\ie}{\textit{i}.\textit{e}.}
\newcommand{\eg}{\textit{e}.\textit{g}.}

\definecolor{darkmidnightblue}{rgb}{0.0, 0.2, 0.4}
\definecolor{aogreen}{rgb}{0.0, 0.5, 0.0}

\definecolor{brickred}{rgb}{0.8, 0.25, 0.33}

\newtheorem{lemma}{Lemma}
\newtheorem{theorem}{Theorem}

\newtheorem{definition}{Definition}

\newtheorem{remark}{Remark}

\title{Approximation Rates in Besov Norms and Sample-Complexity of Kolmogorov-Arnold Networks with Residual Connections}

\author{Anastasis Kratsios\thanks{email: kratsioa@mcmaster.ca}, Bum Jun Kim\thanks{email: bumjun.kim@weblab.t.u-tokyo.ac.jp}, and Takashi Furuya\thanks{email: takashi.furuya0101@gmail.com}}

\usepackage{fancyhdr}
\date{}

\definecolor{britishracinggreen}{rgb}{0.0, 0.26, 0.15}
\definecolor{darkcyan}{rgb}{0.0, 0.55, 0.55}
\definecolor{MidnightBlue}{RGB}{25,25,112}
\definecolor{MidnightBlueComplementingGreen}{RGB}{25,112,25}
\definecolor{MidnightBlueComplementingPurple}{RGB}{112,25,112}
\definecolor{MidnightBlueComplementingRed}{RGB}{112,25,69}
\definecolor{WowColor}{rgb}{.75,0,.75}
\definecolor{MildlyAlarming}{rgb}{0.85,0.25,0.1}
\definecolor{SubtleColor}{rgb}{0,0,.50}
\definecolor{antiquefuchsia}{rgb}{0.57, 0.36, 0.51}
\definecolor{fashionfuchsia}{rgb}{0.96, 0.0, 0.63}
\definecolor{jade}{rgb}{0.0, 0.66, 0.42}
\definecolor{caribbeangreen}{rgb}{0.0, 0.8, 0.6}
\definecolor{aquamarine}{rgb}{0.5, 0.8, 0.85}
\definecolor{lightseagreen}{rgb}{0.13, 0.7, 0.67}
\definecolor{darkgreen}{rgb}{0.0, 0.2, 0.13}
\definecolor{darkspringgreen}{rgb}{0.09, 0.45, 0.27}
\definecolor{forestgreenweb}{rgb}{0.13, 0.55, 0.13}
\definecolor{attentioncolor}{RGB}{152,90,81}
\definecolor{burgred}{RGB}{40,3,22}
\definecolor{AnnieGreen}{RGB}{17,123,92}
\definecolor{Turquoise}{RGB}{64,224,208}

\definecolor{darkjade}{RGB}{0,122,84}
\definecolor{Window1}{RGB}{92,150,31}%
    \definecolor{Window1dark}{RGB}{41,67,13}%
\definecolor{Window2}{RGB}{255,168,28}
    \definecolor{Window2dark}{RGB}{114,75,12}
\definecolor{Window3}{RGB}{255,96,33}
    \definecolor{Window3dark}{RGB}{97,36,12}
\definecolor{InputColor}{RGB}{20,255,177}
    \definecolor{InputColorlight}{RGB}{222,237,229}

\definecolor{RedAlizarin}{rgb}{0.82, 0.1, 0.26}
\usepackage{soul}

\usepackage[colorinlistoftodos]{todonotes}

\NewDocumentCommand{\Takashi}{mo}{
    \IfValueF{#2}{
                        {{
                            \textcolor{magenta}{ 
                            [\textbf{Takashi:}
                            \textit{{#1}}]
                            }
                        }}
        }
    \IfValueT{#2}{
                        \marginnote{{\scriptsize
                            \textcolor{magenta}{ 
                            \textbf{T:}
                            \textit{{#1}}
                            }
                        }}
        }
                    }

\NewDocumentCommand{\AK}{mo}{
    \IfValueF{#2}{\hfill\\
                        {{
                            \textcolor{darkmidnightblue}{ 
                            \textbf{AK:}
                            \textit{{#1}}
                            }
                        }}
        }
    \IfValueT{#2}{
                        \marginnote{{\scriptsize
                            \textcolor{darkmidnightblue}{ 
                            \textbf{AK:}
                            \textit{{#1}}
                            }
                        }}
        }
                    }

\newcommand\numberthis{\addtocounter{equation}{1}\tag{\theequation}}

\begin{document}

\maketitle

\begin{abstract}
	Inspired by the Kolmogorov–Arnold superposition theorem, Kolmogorov–Arnold Networks (KANs) have recently emerged as an improved backbone for most deep learning frameworks, promising more adaptivity than their multilayer perceptron (MLP) predecessor by allowing for trainable spline-based activation functions. In this paper, we probe the theoretical foundations of the KAN architecture by showing that it can optimally approximate any Besov function in $B^{s}_{p,q}(\mathcal{X})$ on a bounded open, or even fractal, domain $\mathcal{X}$ in $\mathbb{R}^d$ at the optimal approximation rate with respect to any weaker Besov norm $B^{\alpha}_{p,q}(\mathcal{X})$; where $\alpha < s$. We complement our approximation result with a statistical guarantee by bounding the pseudodimension of the relevant class of Res-KANs. As an application of the latter, we directly deduce a dimension-free estimate on the sample complexity of a residual KAN model when learning a function of Besov regularity from $N$ i.i.d.\ noiseless samples, showing that KANs can learn the smooth maps which they can approximate. 
\end{abstract}

\section{Introduction}
\label{s:Intro}
Many contemporary deep learning backbones leverage trainable activation functions, evolving from their traditional multi-layer perceptron (MLP) predecessors, to achieve superior adaptivity or training stability. Standard examples range from the SwiGLU~\cite{shazeer2020glu} activation in transformer networks to the adaptive activations in Kolmogorov-Arnold networks~\cite{KANS_OG_2025} typically realized as a linear combination of B-splines, see \eg,~\cite{deBookSpline}, and some interesting theoretical thought experiments are nested neural networks~\cite{zhang2022neural}. In this paper, we focus on models of the latter type with an additional residual connection between layers, as introduced in the ResNet architecture~\cite{he2016deep}, guided by contemporary deep learning wisdom which recognizes the practical benefits of skip connections and has seen a widespread incorporation in contemporary deep learning from GNNs~\cite{borde2024scalable} and transformers~\cite{vaswani2017attention}. Residual connections are theoretically founded and are known to positively regularize a deep learning model's loss landscape~\cite{riedi2023singular}; they allow for narrower networks to maintain their universality~\cite{lin2018resnet} compared to networks without skip connections~\cite{parkminimum,hwang2023minimum}, and they often have no negative drawback on their approximation rates of these models~\cite{gribonval2022approximation}. We call these KANs augmented by a residual connection \textit{Res-KANs}.

Much of the theoretical support for KANs is rooted in the Kolmogorov-Arnold representation theorem; see \eg,~\cite{MR372134} for an optimized formulation. These results either provide approximation guarantees for functions of a composition form~\cite{KANS_OG_2025}, similarly to the composition sparsity~\cite{mhaskar2020analysis,cheridito2021efficient} literature, they offer approximation guarantees in the $L^p$-norms~\cite{wang2025on} by encoding ReLU$^k$-neural network structure and subsequently transferring their $L^p$-type approximation guarantees~\cite{belomestny2023simultaneous,yang2024optimal,furuya2024simultaneously,mao2024approximation} or they show that there exist activation functions which allow them to realize a Kolmogorov-Arnold-type superposition representation of any function~\cite{mhaskar2020analysis} which is reminiscent of the super-expressive activation function literature~\cite{yarotsky2021elementary,zhang2022deep,jiao2023deep,wang2025don}.

Though these results support the expressive power of the KAN paradigm, their applicability to partial differential equation-type (PDE) problems, \eg, those arising in physics, engineering, biology, optimal control, or finance, can still be limited. This is because, in PDE applications, one needs to approximate the function itself but also typically converges the AI model toward the higher-order derivatives of the target functions.

\paragraph{Why Is Studying KANs from the Besov Spaces Lens Natural?}
Given that most KANs are constructed using splines—and that spline-type wavelets are known to characterize Besov spaces (see, \eg,~\cite{DeVorePopov_1988InerpolationofBesovSpaces,ChuiWang_1992_SplineWavelets}). Thus, it is most natural to study KANs within the framework of Besov spaces.

\textbf{Contribution:}
This paper precisely proves that Res-KANs are efficient approximators of functions on Besov spaces, over any reasonable bounded Lipschitz, or even on any compact fractal domain in $\mathbb{R}^d$ (Theorem~\ref{thrm:Main_Approximation}). Importantly, with downstream PDE applications in mind, our approximation guarantees are not (only) in the uniform or $ L^p$ type but in higher-order Besov norms, which are just arbitrarily weaker than the target function's regularity.

We complement our main approximation result with a PAC-learnability guarantee (Theorem~\ref{thm:Main_Generalization}) showing that KANs can, indeed, learn Besov functions of high regularity from noiseless training data. Moreover, the sample complexity is not cursed by dimensionality in the high-smoothness regime.

\subsection{Related Works}
In~\cite[Theorem 3.1]{wang2025on}, the authors establish approximation rates for KANs in the uniform norm, \ie, the $C^s(\Omega) = W^{s,\infty}(\Omega)$ norm—on bounded open domains $\Omega$, but do not provide any statistical guarantees. In contrast, our results measure approximation error in the more general \textit{Besov norms} $B_{p,q}^s(\Omega)$, which subsume the Sobolev norms $W^{s,p}(\Omega)$ when $p=q<\infty$. Theorem~\ref{thrm:Main_Approximation}).
Though there is a relationship between KANs and $\operatorname{ReLU}^k$-MLPs deduced in~\cite[Theorem 3.2]{wang2025on} and there are Sobolev approximation guarantees for $\operatorname{ReLU}^k$-MLPs obtained in~\cite{mao2024approximation}, again both available results only hold for bounded \textit{open} domains, and our results hold for non-open sets of possibly integer dimension; additionally, our approximation error is quantified in the interpolating Besov norms and not only the coarser Sobolev norms. In this way, approximation guarantees (Theorem~\ref{thrm:Main_Approximation}) generalize the available comparable results in a number of ways.

Moreover, unlike~\cite{wang2025on}, we derive \textit{statistical guarantees} for KANs via novel pseudodimension bounds (Lemma~\ref{lem:PseudoDimBound}). In this sense, our contribution goes significantly beyond the existing approximation theory by providing the \emph{first} statistical guarantees for KANs in the literature. The relationship between the approximation results in~\cite{wang2025on} and ours is therefore indirect and fundamentally different in scope.

\paragraph{Organization} Our paper is organized as follows. Section~\ref{s:Background} overviews the necessary background required to formulate our main results; this includes background on cardinal B-splines, Res-KANs, and on Besov spaces over bounded open Lipschitz domains, and on Ahlfors-regular fractal domains, in $\mathbb{R}^d$.
Section~\ref{s:Main_Results} contains our main results, namely our two approximation theorems (Theorem~\ref{thrm:Main_Approximation}) and our pseudodimension bound (Lemma~\ref{lem:PseudoDimBound}). All proofs and additional background, \eg, fat shattering and pseudodimension, are only needed for proof details and are relegated to our appendices.

\section{Background}
\label{s:Background}

We now present the necessary background to formulate our main result.
\subsection{Cardinal B-Splines}
\label{s:Background__ss:ResKANs_}
Kolmogorov-Arnold Networks (KANs) extend the multilayer perceptron (MLP) architecture by allowing its otherwise fixed and, thus, rigid, univariate activation functions to adapt to the individual local structures of the function being approximated. This adaptivity is realized by replacing each activation function with a cardinal B-spline $\mathcal{N}_k$ of order $k$ for appropriately learned $k$.

\begin{figure}[H]
	\centering
	\begin{minipage}[t]{0.4\textwidth}
		\vspace{0pt} 
		\includegraphics[width=.9\linewidth]{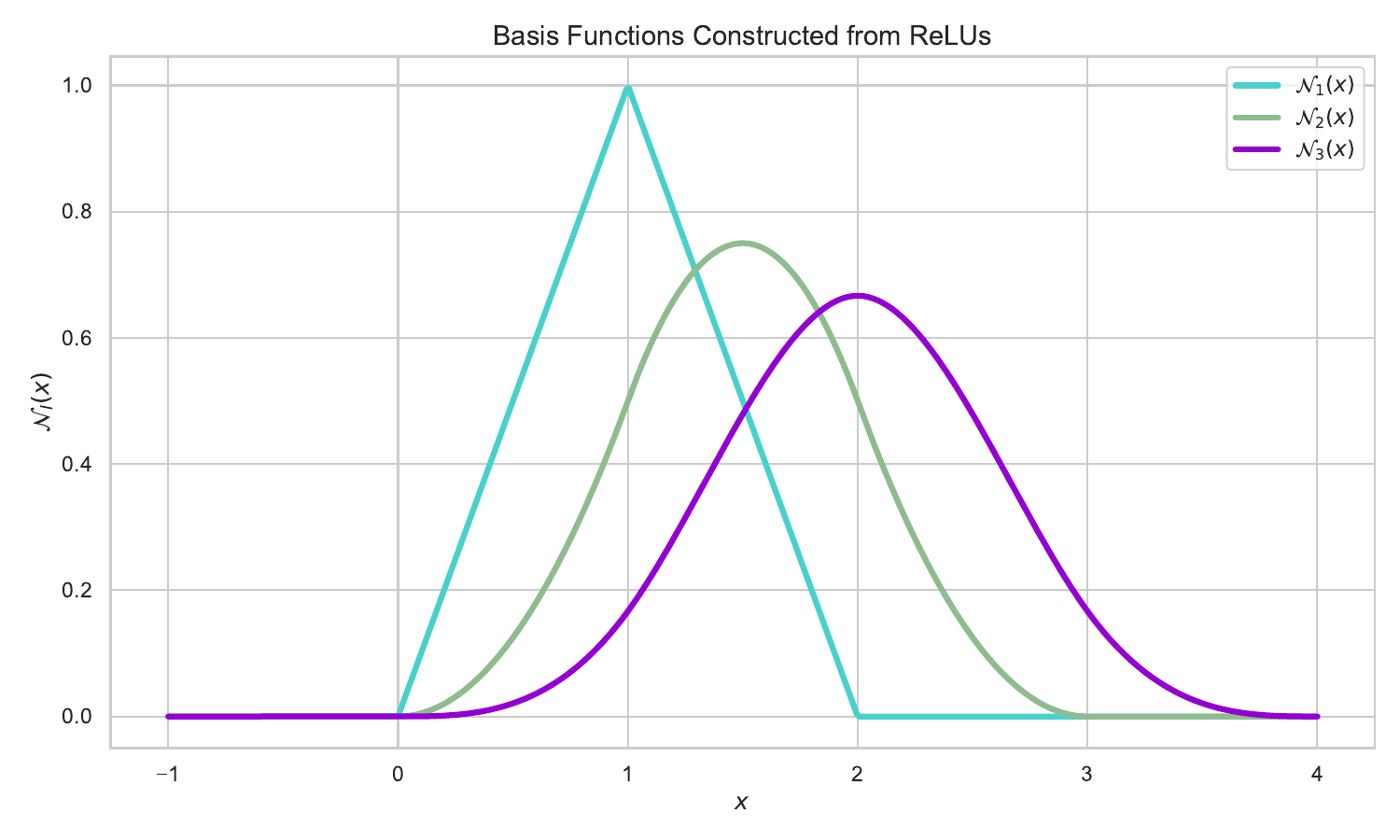}
	\end{minipage}
	~
	\begin{minipage}[t]{0.55\textwidth}
		\vspace{0pt} 
		\begin{adjustbox}{max width=\textwidth}
			\small
			\begin{tabular}{cl}
				\toprule
				\multicolumn{2}{c}{\textbf{First Cardinal $B$-Splines}}                                                                                                                 \\
				\midrule
				$I$ & $\mathcal{N}_I$                                                                                                                                                   \\
				$0$ & $\mathbf{1}_{[0,1)}(x)$                                                                                                                                           \\
				$1$ & $\operatorname{ReLU}(x) - 2\operatorname{ReLU}(x{-}1) + \operatorname{ReLU}(x{-}2)$                                                                               \\
				$2$ & $\frac{\operatorname{ReLU}(x)^2}{2} - \frac{3\operatorname{ReLU}(x{-}1)^2}{2} + \frac{3\operatorname{ReLU}(x{-}2)^2}{2} - \frac{\operatorname{ReLU}(x{-}3)^2}{2}$ \\
				\bottomrule
			\end{tabular}
		\end{adjustbox}
	\end{minipage}
	\caption{The cardinal $B$-splines of orders $I=0,1$, and $2$.}
	\label{fig:SplineIllustrations}
\end{figure}

As shown, for example, in~\cite[Equation (4.28)]{MhaskarMichelli_ApprxSuperpositionRBFs_AAM1992}, for any $I\in \mathbb{N}_+$, the \textit{cardinal B-spline} of order $I$ with knots on $0,\dots,I+1$ is given by
\begin{equation}
	\label{lem:MichelliRepresentation}
	\mathcal{N}_I(x)
	=
	\sum_{j=0}^{I+1}
	\frac{(-1)^j \binom{I+1}{j} }{I!}
	\,
	\operatorname{ReLU}(x-j)^I
\end{equation}
for each $x\in \mathbb{R}$.
Figure~\ref{fig:SplineIllustrations} illustrates the cardinal B-splines with $I=0,1$, and $2$, as detailed in the following examples.
\subsection{Residual Kolmogorov-Arnold Networks (Res-KANs)}
\label{s:Background__ss:ResKANs}
The core idea behind the KAN is to allow the activation function to be trainable, where we enable mixtures of different spline degrees via sending any $x\in \mathbb{R}$ to
\begin{equation}
	\label{eq:KAN_r}
	\sigma_{\beta:I}(x)
	=
	\sum_{i=0}^{I}
	\beta_i
	\,
	\mathcal{N}_i(x)
\end{equation}
where the trainable parameter $\beta\in \mathbb{R}^{I+1}$. In a KAN, each activation function acts componentwise but with trainable parameters depending on the neuron it is activating; that is, for each $d\in \mathbb{N}_+$, every $x\in \mathbb{R}^d$, and every $\mathbf{\beta}\eqdef (\beta_1,\dots,\beta_k)\in \mathbb{R}^{(I+1)\times d}$ we have
\begin{equation}
	\label{eq:componentwise_action}
	\sigma_{\mathbf{\beta}:I}\bullet(x)
	\eqdef
	\big(
	\sigma_{\beta_j:I}(x_j)
	\big)_{j=1}^d
	.
\end{equation}
We now introduce the core of our \textit{residual} KAN networks, which ensure that no signal is lost during activation by incorporating additional \textit{residual connections}, which have also become standard in modern AI as they additionally stabilize training dynamics by preserving gradient flow by regularizing the neural network's loss landscape~\cite{riedi2023singular}, and avoid vanishing gradient problems that can be caused by normalizing layers. As in~\cite{acciaio2024designing}, we allow first any residual connections to be skipped or focused on, if need be, via a trainable gating mechanism
\begin{equation}
	\label{eq:norm_res_KAN_layers}
	\mathcal{L}(x|A,b,\beta,G:I)
	\eqdef
	\underbrace{
		\sigma_{\mathbf{\beta}:I}\bullet (Ax+b)
	}_{\text{KAN Layer}}
	+
	\underbrace{
		G x
	}_{\text{Residual Connection}}
\end{equation}
here the layer input and output dimensions respectively as $d_{in},d_{out}\in \mathbb{N}_+$
$A,G$ are $d_{out}\times d_{in}$-matrices with a matrix $G$ \textit{diagonal}; meaning $G_{i,j}=0$ if $i\neq j$ (not $G$ need not be a square matrix), $\beta$ is a $(I+1) \times d_{\ell+1}$ matrix, $b\in \mathbb{R}^{d_{\ell+1}}$.

Although the composition of these KAN layers defines a meaningful function, that object may not be very regular, in the sense that it may have few higher-order derivatives, making it problematic for PDE applications where a high degree of regularity is required. There are two solutions: 1) the $\beta_i$ coefficients are all zero for small $i$, or 2) we incorporate a simple smoothing layer at the output. We opt for the second option, as then any function implemented by our \textit{smoothed residual KANs} is necessarily smooth; \ie, infinitely differentiable.

Besides the usual cardinal $B$-splines $\{\mathcal N_0,\dots,\mathcal N_I\}$, we also treat the ReLU $\mathcal N_{\mathrm R}(t)\equiv\relu(t)$ and its square $\mathcal N_{\mathrm R^{2}}(t)\equiv\relu(t)^2$ as additional degree-1 and degree-2 basis functions. Including these two functions into the dictionary slightly enlarges the basis available for Res-KANs.

\begin{definition}[Residual KANs (Res-KANs)]
	\label{defn:SmoothedKAns}
	Let $d,D,I\in \mathbb{N}_+$ and $\alpha>0$. Here, we use $\{\mathcal{N}_{0},\dots,\mathcal{N}_{I},\mathcal{N}_{\mathrm R},\mathcal{N}_{\mathrm{R}^{2}}\}$, where $\mathcal{N}_{i}$ is the cardinal $B$-spline of order $i$ and $\mathcal{N}_{\mathrm R}(t)=\max\{0,t\}$, $\mathcal{N}_{\mathrm{R}^{2}}(t)=\max\{0,t\}^{2}$. 
	A residual Kolmogorov-Arnold network (Res-KAN) is a function $\hat{f}:\mathbb{R}^d\to \mathbb{R}^D$ with representation
	\begin{equation}
		\label{eq:DEF_SRKs}
		\begin{aligned}
			\hat{f}    & =
			A^{(L+1)} f^{(L)}+b^{(L+1)}
			\\
			f^{(l)}    & = \mathcal{L}(f^{(l-1)}|A^{(l)},b^{(l)},\beta^{(l)},G^{(l)}:I+2)
			\mbox{ for } l=1,\dots,L
			\\
			f^{(0)}(x) & = x
		\end{aligned}
	\end{equation}
	where, for $l=1,\dots,L$, $A^{(l)}$ and $G^{(l)}$ are $d_{l+1}\times d_l$ matrices with $G$ \textit{diagonal}, $b\in \mathbb{R}^{d_{l+1}}$, $d_0,\dots,d_{L+1}\in \mathbb{N}_+$ with $d_0=d$ and $d_{L+1}=D$; furthermore, we take $\beta^{(l)}\in\mathbb{R}^{(I+3)\times d_{l+1}}$ and apply the sparsity rule only to its first $I+1$ rows, for which $\beta^{(l)}$ satisfies the \textit{sparsity} pattern (ensuring smoothness)
	\begin{equation}
		\label{eq:DEF_SRKs__sparsity}
		\beta^{(l)}_{i,j} = 0
		\mbox{ for }
		i < \lceil \alpha \rceil
		.
	\end{equation}
	We denote the class of all Res-KANs with $L$ hidden layers, width $W\eqdef \max_{l=1,\dots,L+1}\,d^{(l)}$, adaptivity parameter $I$, and smoothness parameter $\alpha$ by $\operatorname{Res-KAN}_{L,W}^{I,\alpha}(\mathbb{R}^d,\mathbb{R}^D)$.
\end{definition}

\subsection{Besov Spaces}
\label{s:Background__ss:Besov}
We begin with the definition of the Besov spaces on any non-empty open domain $\Omega$ in $\mathbb{R}^d$.
\paragraph{Besov Spaces on Euclidean Domains}
Let $0<p\le \infty$, and $0<\alpha$. The modulus of smoothness of order $\alpha $ of an $f\in L^p(\Omega)$, defined with respect to the restricted Lebesgue measure on $\Omega$, is defined for each $t>0$
\[
	\omega_{\alpha}(f,t)_p
	\eqdef
	\sup_{\delta\in \mathbb{R}^d\, 0<\|\delta\|\le t}\,
	\|\Delta_{\delta}^{\lceil \alpha\rceil}(f,\cdot)\|_{L^p(\lceil \alpha \rceil h)}
\]
where $\Delta_h^t$ is the $\lceil \alpha\rceil^{th}$ order finite difference operator with step-size $\delta\in \mathbb{R}^d\setminus\{0\}$ and $\|g\|^p_{L^p(\lceil \alpha \rceil h)}\eqdef \int |g(u)|^p I_{\lceil \alpha \rceil h}dx <\infty$, where $I_{\lceil \alpha \rceil h}$ is the indicator function of the set $\{u\in \mathbb{R}^d:\, u+\lceil \alpha \rceil h\in \Omega\}$. For any $0<q<\infty$, one of the many equivalent formulations of the Besov space $B_{p,q}^{\alpha}(\Omega)$ is the collection of all $f\in L^p(\Omega)$ for which the following quasi-norm is finite
\begin{equation}
	\label{eq:Besov_Definition}
	\|f\|_{B_{p,q}^{\alpha}(\Omega)}
	\eqdef
	\biggl(
	\int_0^{\infty} \frac{\omega_{\alpha}(f,t)_p^q}{t^{q\alpha+1}} \, dt
	\biggr)^{1/q}
	+
	\|f\|_{L^p(\Omega)}
	.
\end{equation}
We will focus on the following class of domains:
An open set $\Omega$ is called an \textit{extension domain} in $\mathbb{R}^d$ if there are some $\epsilon,\delta>0$ such that for any $x,y\in \Omega$ satisfying
\[
	\|x - y\| \leq \delta
\] there exists a rectifiable curve $\gamma:[0,1]\to \mathbb{R}^d$ of length at most $C_0 \|x - y\|$ with $\gamma(0)=x$ and $\gamma(1)=y$ such that, for any time $t\in [0,1]$
\[
	\inf_{u\in \partial \Omega}\, \|\gamma(t)-u\| \geq
	\epsilon \min\{\|\gamma(t) - x\|, \|\gamma(t) - y\|\}.
\]
We are interested in extension domains due to the extendability of functions therein to all of $\mathbb{R}^d$, while also being able to ensure that their Sobolev regularity is preserved under these extensions; see \eg,~\cite{jones1981quasiconformal,brewster2014extending}. Extension domains, in the above sense, are often referred to as $(\epsilon,\delta)$-domains in harmonic analysis. However, we avoid this terminology here to prevent confusion with its standard use in approximation and learning theory, where $\epsilon$ typically denotes an approximation error and $\delta$ the probability that the approximation is valid.

\paragraph{Besov Spaces on Fractals}
We begin with the definition of a Besov space on an arbitrary well-behaved, possibly fractional, domain $\mathcal{X}$ in some Euclidean space $\mathbb{R}^d$ as introduced in~\cite{JonssonOG_FractalsBesov1984}.
Let $0<n\le d$, and let $\mathcal{H}^n$ denote the $n$-dimensional Hausdorff (outer) measure on $\mathbb{R}^d$.
We denote a closed $\ell^{\infty}$-ball of radius $r>0$ centered at some $x\in \mathbb{R}^n$ by $Q(x,r)\eqdef \{y\in \mathbb{R}^d:\, \|x-y\|_{\infty}\le r\} = [x-r/2,x+r/2]^d$.
A subset $\mathcal{X}\subseteq \mathbb{R}^d$ is called Ahlfors $n$-regular if there are constants $0<c\le C$ such that for all $0<r\le \operatorname{diam}(\mathcal{X})$ and for each $x\in \mathcal{X}$
\begin{equation}
	\label{eq:Ahlfors_RegularityCondition}
	cr^n
	\le
	\mathcal{H}^n\big(
	\mathcal{X}
	\cap
	Q(x,r)
	\big)
	\le
	Cr^n
	.
\end{equation}
\begin{remark}
The dimensionality restriction, which we will consider shortly for Ahlfors regular domains, namely that $d-1<n<d$, directly implies that no $(\epsilon,\delta)$-domain is Ahlfors $n$-regular and vice versa, as every $(\epsilon,\delta)$-domain has integer dimension.
\end{remark} 
For $1\le p<\infty$, we define the \textit{normalized local best polynomial approximation} energy as
\begin{equation}
	\label{eq:normalized_best}
	\mathcal{E}_k(f,Q)_{L^p(\mathcal{X})}
	\eqdef
	\inf_{p \in \mathbb{R}_{k-1}[x_1,\dots,x_d]}\,
	\biggl(
	\frac{1}{\mathcal{H}^n(Q \cap S)}
	\int_{Q \cap S} |f - P|^p \, d\mathcal{H}^n
	\biggr)^{1/p}
\end{equation}
where $\mathcal{Q}=\operatorname{Ball}_{\ell^{\infty}}(x,r)$ for some $x\in \mathcal{X}$ and some $r>0$, $\mathbb{R}_{k-1}[x_1,\dots,x_d]$ is the vector space of polynomials on $x_1,\dots,x_d$ of degree at most $k-1$ with the convention that $\mathbb{R}_{-1}[x_1,\dots,x_d]\eqdef \{0\}$ is the trivial vector space.
\begin{definition}[Besov Space on Ahlfors-Regular Sets]
	\label{def:Besov}
	Let $0<n\le d$, $\mathcal{X}\subseteq \mathbb{R}^d$ be Ahlfors $n$-regular, and let $\alpha > 0$, $1 \leq p, q \leq \infty$.
	The Besov space $B^{\alpha}_{p,q}(\mathcal{X})$ consists of all functions $f \in L^p(\mathcal{X})$ for which the norm
	\[
		\| f \|_{B^{\alpha}_{p,q}(\mathcal{X})} \eqdef \| f \|_{L^p(\mathcal{X})} + \left( \int_0^1 \left( \frac{\| \mathcal{E}_k (f, Q(\cdot, \tau)) \|_{L^u(\mathcal{X})}}{\tau^\alpha} \right)^q \frac{d\tau}{\tau} \right)^{\frac{1}{q}},
	\]
	is finite, where $1 \leq u \leq p$ and $k$ is an integer such that $\alpha < k$.
\end{definition}
Importantly, by~\cite[Theorem 3.6]{Ihnatseyva_AhlforsRegularJFA_2013}, the definition of the Besov space $B^{\alpha}_{p,q}(\mathcal{X})$ above does not depend on the (arbitrary) choice of parameters $k$ and $u$; granted that $k>\alpha$ and $1\le u\le p$.

\section{Main Results}
\label{s:Main_Results}
We are now in a position to state our main approximation guarantee.
\begin{theorem}[Approximation Guarantees for Res-KANs in Besov Norm]
	\label{thrm:Main_Approximation}
	Let $0<\alpha<s<\infty$ and $0< p,q<\infty$, $d-1<n< d$, and $\mathcal{X}\subseteq [0,1]^d$ either be (1) an $(\epsilon,\delta)$-domain or (2) Ahlfors $n$-regular for some $d-1<n<d$ and additionally $1\le p,q$. In case (1) set $\alpha^{\star}\eqdef \alpha$, and in case (2) set $\alpha^{\star}\eqdef \alpha-(n-d)/p$.
	For every $f\in B_{p,q}^{\alpha^{\star}}(\mathcal{X})$ and every ``simultaneous approximation error'' $\varepsilon>0$, there exists a Res-KAN $\hat{f}:\mathbb{R}^d\to \mathbb{R}$ such that $\hat{f}\in B_{p,q}^{r}(\mathcal{X})$ satisfying
	\[
		\|f-\hat{f}|_{\mathcal{X}}\|_{B_{p,q}^r(\mathcal{X}))}
		<
		\varepsilon
	\]
	and $\hat{f}$ has width $\mathcal{O}(\varepsilon^{1/((\alpha^{\star} - s)
    )})$, depth $\mathcal{O}(d)$, and $\mathcal{O}\big(d^2 \varepsilon^{1/((\alpha^{\star} - s)
    )}\big)$ 
    non-zero parameters.
\end{theorem}
\begin{proof}
	See Appendix~\ref{s:Proof_Approx}.
\end{proof}

We now complement our approximation guarantees for residual KANs with corresponding learning guarantees. Specifically, given any target function $f \in B_{p,q}^{\alpha}(\mathcal{X})$ of prescribed Besov regularity, we ask: \textit{How many i.i.d.\ samples $N$ are required to ensure that the in-sample performance of a Res-KAN estimator $\hat{f}$ reliably reflects its out-of-sample performance?} This question is posed within the standard nonlinear regression setting, where both random input sampling and additive measurement noise are present.

{We answer this question in the non-linear regression setting, where we wish to estimate data arising from a true (measurable) function $\hat{f}:\mathbb{R}^d\to \mathbb{R}$. The training data consists of random pairs $(X,Y)$ where the random inputs $X$ are drawn from a Borel probability measure $\mathbb{P}_{\operatorname{smpl}}$ on $\mathbb{R}^d$ and the outputs are $f(X)+\nu$, where $\nu$ is random measurement noise independent of $X$ and is drawn with respect to a mean zero Borel probability measure $\mathbb{P}_{\operatorname{err}}$ on $\mathbb{R}$. }
We quantify the true error between any given Res-KAN $\hat{f}$ and any target $f$ by their \textit{true risk} $\mathcal{R}_{\mathbb{P}}(f|\hat{f})$ and their \textit{empirical risk} $\hat{\mathcal{R}}_{\mathbb{P}}^N(f|\hat{f})$ defined by
\[
	\mathcal{R}
    (f|\hat{f})
    =
    \mathbb{E}_{
    \overset{
        X\sim \mathbb{P}_{\operatorname{smpl}}
        }{
        \nu\sim \mathbb{P}_{\operatorname{err}}
        }
    }\big[|\hat{f}(X)
    -
    {(}f(X){+\nu)}
    |\big]
	\mbox{ and }
	\hat{\mathcal{R}}
    ^N(f|\hat{f})=\frac1{N}\sum_{n=1}^N 
    \,
        |
            \hat{f}(X_n)-{(}f(X_n){+\nu_n)}
        |
\]
where $X_1,\dots,X_N$ are i.i.d.\ samples from $\mathbb{P}_{\operatorname{smpl}}$ and $\nu_1,\dots,\nu_N$ are i.i.d.\ samples from $\mathbb{P}_{\operatorname{err}}$.
\begin{theorem}[{Sample Complexity for Res-KANs When Learning From Noisy Besov Data}]
	\label{thm:Main_Generalization}
	Suppose that $\mathcal{X}$ is a Lipschitz domain, let $1 \le \tau \le \infty$, $1 \le p, q < \infty$, and $\alpha > 
    \big(
            {(}
                d
            {+1)}
    (1/p - 1/\tau)\big)_+$ and let $L,I,W\in \mathbb{N}_+$. For every Borel probability measure $\mathbb{P}_{\operatorname{smpl}}\in \mathcal{P}(\mathcal{X})$%
    {, each centered Borel probability measure 
    $\mathbb{P}_{\operatorname{smpl}}\in \mathcal{P}(\mathbb{R})$, }
    each training set of i.i.d.\ samples $(X_1,Y_1),\dots,(X_N,Y_N)\sim \mathbb{P}_{\operatorname{smpl}}\otimes \mathbb{P}_{\operatorname{err}}$ every approximation error $\varepsilon>0$ and every failure probability $0<\delta \le 1$ if
	\[
		N
		\in
		\mathcal{O}\Bigl(
		\epsilon^{-2 - 
            {(}
                d
            {+1)}
        /\alpha}\,(\ln(1/\varepsilon))^2
		+\epsilon^{-2}\ln(1/\delta)
		\Bigr).
	\]
	then
	\[
		\mathbb{P}\Biggl(
		\sup_{f,\hat{f}}
		\,
		\Big|
		\mathcal{R}
        (f|\hat{f})
		-
		\hat{\mathcal{R}}
        ^N(f|\hat{f})
		\Big|
		\le
		\varepsilon
		\Biggr)
		\ge
		1
		-
		e^{
				-cN\epsilon^{2}
				+
                    {(}
                        d
                    {+1)}
				\ln^{2}
				\Big(
				\tfrac{
                    {(}
                        d
                    {+1)}
                }{\epsilon}
				\Big)
			}
	\]
	where the supremum is taken over all $\hat{f}\in \operatorname{Res-KAN}_{L,W}^{I,\alpha}(\mathbb{R}^d,\mathbb{R})$ and all $f\in B_{p,q}^{\alpha}(\mathcal{X})$ with Besov norm $\|f\|_{B_{p,q}^{\alpha}(\mathcal{X})}\le 1$.
\end{theorem}
\begin{proof}
	See Appendix~\ref{s:Proof_Generalization}.
\end{proof}

{
\subsection{Implications: KANs vs.\ Other Deep Learning Alternatives}
It is known that ReLU MLPs~\cite{suzuki2018adaptivity} and Transformers~\cite{kim2024transformers} achieve optimal approximation rates (in the sense of constructive approximation theory) and optimal sample complexity (via empirical process theory). It, thus, cannot be improved in these respects. Since our results show that KANs match these same rates, we conclude that KANs are likewise optimal from both constructive approximation (Theorem~\ref{thrm:Main_Approximation}) and statistical learning (Theorem~\ref{thm:Main_Generalization}) perspectives. Accordingly, KANs do not appear to offer any advantage over MLPs or Transformers from these theoretical standpoints.
\hfill\\
Any empirical advantage observed for KANs over classical models must therefore arise from other sources, such as their inductive bias when trained via SGD-type algorithms or the stability of their training dynamics. The main practical takeaway from our work is that identifying the advantage of KANs will require going beyond approximation or generalization guarantees and focusing instead on their behavior during optimization.
}

\section{Conclusion}
\label{s:Conclusion}
In this paper, we establish the theoretical foundations of Kolmogorov–Arnold Networks (KANs), showing that they can approximate any Besov function on a bounded or even fractal domain at the optimal rate in Theorem~\ref{thrm:Main_Approximation}. We also provide a dimension-free sample complexity bound for learning such functions with a residual KAN model in Theorem~\ref{thm:Main_Generalization}. Due to the deep connection between Besov spaces and splines~\cite{DeVorePopov_1988InerpolationofBesovSpaces,ChuiWang_1992_SplineWavelets}, we believe that this is a very natural setting to quantify the power of Kolmogorov-Arnold networks. Our results relied on a KAN build which incorporated residual connections, aligning with modern deep learning practices. Simple toy experiments further confirm that adding residual connections does not degrade performance, and the KAN retains accuracy similar to that of its non-residual counterpart.

\section*{Acknowledgements}
The authors would like to thank \href{https://sites.google.com/view/yaoliding}{Yao Liding} for their very helpful feedback on Sobolev extension operators over Euclidean domains.

A.\ Kratsios acknowledges financial support from an NSERC Discovery Grant No.\ RGPIN-2023-04482 and No.\ DGECR-2023-00230. They both also acknowledge that resources used in preparing this research were provided, in part, by the Province of Ontario, the Government of Canada through CIFAR, and companies sponsoring the Vector Institute\footnote{\href{https://vectorinstitute.ai/partnerships/current-partners/}{https://vectorinstitute.ai/partnerships/current-partners/}}. T.\ Furuya was supported by JSPS KAKENHI Grant Number JP24K16949, 25H01453, JST CREST JPMJCR24Q5, JST ASPIRE JPMJAP2329.

\appendix
\section{{Proof of Main Approximation Result}}
\label{s:Proof_Approx}
We now prove Theorem~\ref{thrm:Main_Approximation}. We first verify that residual KANs can locally implement the multiplication operation. The next lemma serves as an exact version of the well-known fact that $\tanh$-MLPs~\cite{de2021approximation}, ReLU MLPs~\cite{yarotsky2018optimal}, and, more generally, MLPs with smooth activation functions~\cite{kidger2020universal,kratsios2022universal,zhang2024deep}, can \textit{approximately} implement the $d$-fold multiplication operation on arbitrarily large hypercubes. This next result shows that residual KANs can \textit{exactly} implement the $d$-fold multiplication operation, locally, on arbitrarily large hypercubes.

\begin{lemma}[Exact Multiplication on Arbitrarily-Large Hypercubes]
	\label{lem:MultNet}
	For every $d\in \mathbb{N}_+$ and each $M>0$ there exists a Res-KAN $\times^2_d:\mathbb{R}^d\to \mathbb{R}$ satisfying: for each $x\in [-M,M]^d$
	\[
		\times^2_d(x) = \prod_{i=1}^d\, x_i
	\]
	Moreover, $\times^2_d$ has depth $2(d-1)$, width at most $d+4$, and the number of non-zero parameters $d^2+33d-36$.
\end{lemma}

\begin{proof}
	\textbf{Step 0 – Square gadget.}
	Consider
	\[
		q(t):=\relu(t)^{2}+\relu(-t)^{2}=t^{2}\qquad(t\in\R).
	\]
	This square gadget can be implemented by
	\begin{itemize}
		\item Affine matrix \( A_{\text{square}} =\begin{pmatrix}1\\ -1\end{pmatrix}\in\R^{2\times1} \) with $2$ non-zero affine weights, with \( b_{\text{square}} =\begin{pmatrix}0\\ 0\end{pmatrix} \).
		\item Activation coefficients using two columns of \( (0, \dots, 0, 0, 1)^\top \) for $\beta$, yielding $2$ non-zero $\beta$-entries.
		\item Head weights $(1,1)$ to sum the two channels.
	\end{itemize}

	\medskip
	\textbf{Step 1 – Two-input multiplier.}
	For $u,v\in[-M,M]$, we have
	\[
		uv=\tfrac12\bigl[q(u+v)-q(u)-q(v)\bigr].
	\]
	Motivated by this observation, we build $q(u+v)$, $q(u)$, and $q(v)$ to implement the two-input multiplier. Specifically, we use an affine matrix
	\[
		A_{\text{mult}}
		=\begin{pmatrix}
			1  & 1  \\[-2pt]
			-1 & -1 \\[-2pt]
			1  & 0  \\[-2pt]
			-1 & 0  \\[-2pt]
			0  & 1  \\[-2pt]
			0  & -1
		\end{pmatrix}\in\R^{6\times 2},
	\]
	which has $8$ non-zero entries. We then apply the square gadget with $6$ activation coefficients and a head with weights of \(A_{\text{head}} = \tfrac12(1, 1,-1,-1,-1,-1), \) amounting to $6$ head weights. Thus, the block uses \(8+6+6=20\) non-zero parameters and temporarily occupies $6$ auxiliary channels.

	\medskip
	\textbf{Step 2 – Implementing $d$-fold multiplication.}
	Consider $p_{1}=x_{1}$ and
	\[
		p_{l}:=\tfrac12\bigl[(p_{l-1}+x_{l})^{2}-p_{l-1}^{2}-x_{l}^{2}\bigr],
	\]
	for $l=2,\dots,d$. Induction gives $p_{l}=\prod_{i=1}^{l}x_{i}$, hence $p_{d}=\prod_{i=1}^{d}x_{i}$.

	Now, we implement a two-input multiplier using the first two elements. Specifically, for the very first layer with input $(x_1, x_2, \dots, x_d) \in \R^{d}$, we aim to build a two-input multiplier on $x_1$ and $x_2$ and pass the remaining elements. Note that the two-input multiplier uses an affine matrix and a head matrix, which requires two Res-KAN layers. In this unit, the first Res-KAN layer computes six terms with squared ReLU and passes the remaining $(d-2)$ elements using
	\[
		A_{1,\text{first}}
		=\begin{pmatrix}
			A_\text{mult}               & \mathbf{0}_{6 \times (d-2)}     \\
			\mathbf{0}_{(d-2) \times 2} & \mathbf{0}_{(d-2) \times (d-2)}
		\end{pmatrix}\in\R^{(d+4)\times d},
	\]
	\[
		G_{1,\text{first}}
		=\begin{pmatrix}
			\mathbf{0}_{6 \times 2}     & \mathbf{0}_{6 \times (d-2)}     \\
			\mathbf{0}_{(d-2) \times 2} & \mathbf{I}_{(d-2) \times (d-2)}
		\end{pmatrix}\in\R^{(d+4)\times d},
	\]
	with the six activation coefficients that correspond to $\mathcal{N}_{\mathrm{R}^{2}}(t)=\max\{0,t\}^{2}$.

	From the six terms with squared ReLU and $(d-2)$ remaining elements, the second Res-KAN layer applies
	\[
		A_{1,\text{second}}
		=\begin{pmatrix}
			A_\text{head}               & \mathbf{0}_{1 \times (d-2)}     \\
			-A_\text{head}              & \mathbf{0}_{1 \times (d-2)}     \\
			\mathbf{0}_{(d-2) \times 6} & \mathbf{0}_{(d-2) \times (d-2)}
		\end{pmatrix}\in\R^{d\times (d+4)},
	\]
	\[
		G_{1,\text{second}}
		=\begin{pmatrix}
			\mathbf{0}_{2 \times 6}     & \mathbf{0}_{2 \times (d-2)}     \\
			\mathbf{0}_{(d-2) \times 6} & \mathbf{I}_{(d-2) \times (d-2)}
		\end{pmatrix}\in\R^{d \times (d+4)},
	\]
	with the six activation coefficients that correspond to $\mathcal{N}_{\mathrm{R}}(t)=\max\{0,t\}$. Note that the second KAN layer here outputs $(\relu(x_1 x_2), \relu(-x_1 x_2), 0, \dots, 0) \in \R^{d}$ and the residual connection outputs $(0, 0, x_3, \dots, x_d) \in \R^{d}$; thus, the second Res-KAN layer yields the sum $(\relu(x_1 x_2), \relu(-x_1 x_2), x_3, \dots, x_d) \in \R^{d}$. We then merge $\relu(x_1 x_2)$ and $\relu(-x_1 x_2)$ in the first layer of the subsequent unit.

	Now consider the $l$th unit with $l \geq 2$. The first Res-KAN layer computes six terms with squared ReLU and passes the remaining $(d-l-1)$ elements using
	\[
		A_{l,\text{first}}
		=\begin{pmatrix}
			A_\text{merge}                & \mathbf{0}_{6 \times (d-l-1)}       \\
			\mathbf{0}_{(d-l-1) \times 3} & \mathbf{0}_{(d-l-1) \times (d-l-1)}
		\end{pmatrix}\in\R^{(d-l+5)\times (d-l+2)},
	\]
	\[
		G_{l,\text{first}}
		=\begin{pmatrix}
			\mathbf{0}_{6 \times 3}       & \mathbf{0}_{6 \times (d-l-1)}       \\
			\mathbf{0}_{(d-l-1) \times 3} & \mathbf{I}_{(d-l-1) \times (d-l-1)}
		\end{pmatrix}\in\R^{(d-l+5) \times (d-l+2)},
	\]
	with the six activation coefficients that correspond to $\mathcal{N}_{\mathrm{R}^{2}}(t)=\max\{0,t\}^{2}$, where
	\[
		A_{\text{merge}}
		=\begin{pmatrix}
			1  & -1 & 1  \\[-2pt]
			-1 & 1  & -1 \\[-2pt]
			1  & -1 & 0  \\[-2pt]
			-1 & 1  & 0  \\[-2pt]
			0  & 0  & 1  \\[-2pt]
			0  & 0  & -1
		\end{pmatrix}\in\R^{6\times 3}.
	\]
	We exploit the fact that $\relu(x)-\relu(-x)=x$ here to merge the two ReLUs, implementing the two-input multiplier intermediately. Similarly, from the six terms with squared ReLU and remaining $(d-l-1)$ elements, the second Res-KAN layer applies
	\[
		A_{l,\text{second}}
		=\begin{pmatrix}
			A_\text{head}                 & \mathbf{0}_{1 \times (d-l-1)}       \\
			-A_\text{head}                & \mathbf{0}_{1 \times (d-l-1)}       \\
			\mathbf{0}_{(d-l-1) \times 6} & \mathbf{0}_{(d-l-1) \times (d-l-1)}
		\end{pmatrix}\in\R^{(d-l+1)\times (d-l+5)},
	\]
	\[
		G_{l,\text{second}}
		=\begin{pmatrix}
			\mathbf{0}_{2 \times 6}       & \mathbf{0}_{2 \times (d-l-1)}       \\
			\mathbf{0}_{(d-l-1) \times 6} & \mathbf{I}_{(d-l-1) \times (d-l-1)}
		\end{pmatrix}\in\R^{(d-l+1) \times (d-l+5)},
	\]
	with the six activation coefficients that correspond to $\mathcal{N}_{\mathrm{R}}(t)=\max\{0,t\}$. The second Res-KAN layer here yields $(\relu(p_{l+1}), \relu(-p_{l+1}), x_{l+2}, \dots, x_d) \in \R^{(d-l+1)}$.

	We repeat this unit for $l=2, 3, \dots, d-1$, yielding an output of $(\relu(p_d), \relu(-p_d))$. Finally, we use the head of Res-KAN as $A^{(L+1)}=(1, -1)$ to obtain $p_d$, implementing the $d$-fold multiplication as desired.

	\medskip
	\textbf{Step 3 – Parameter count.} Note that the $l$th unit contains two Res-KAN layers. Counting up to $d-1$, the total depth amounts to $2(d-1)$. The maximum width is $d+4$ at the first unit.

	Now, we count the number of non-zero parameters. For the very first unit, the first Res-KAN layer contains $8$ non-zero parameters in the affine matrix, $6$ in the activation coefficients, and $d-2$ in the $G$ matrix, whereas its second Res-KAN layer includes $12$ in the two head matrices, $6$ in the activation coefficients, and $d-2$ in the $G$ matrix; the total number amounts to $2d+28$. For the $l$th unit, the first Res-KAN layer contains $12$ in the affine matrix, $6$ in the activation coefficients, and $d-l-1$ in the $G$ matrix; the same goes for its second Res-KAN layer; the total number amounts to $2d-2l+36$ for the $l$th unit. Finally, we add the two that are used in the head of Res-KAN. Counting all non-zero parameters amounts to $d^2+33d-36$.
\end{proof}

Together, we may approximately implement a basic component of a spline-type multi-resolution analysis (MRA); see~\cite{mallat1989multiresolution}, a type of characterization of the Besov space $B^{\alpha}_{p,q}([0,1]^d)$ as derived in~\cite{DeVoreSharpleyBesovDomains1993}.
Returning to our cardinal B-splines above, expressed in~\eqref{lem:MichelliRepresentation}, we consider the $d$-fold tensor product of such splines $\mathcal{N}_{d:I}:\mathbb{R}\to \mathbb{R}$ given for each $x\in \mathbb{R}^d$ by
\begin{equation}
	\label{eq:splines_tensor}
	\mathcal{N}_{d:I}(x)
	\eqdef
	\prod_{k=1}^d\,\mathcal{N}_{I}(x_k)
	.
\end{equation}
Next, for each $k\in \mathbb{N}_+$ let $\mathbb{D}_k\eqdef \{
	\prod_{i=1}^d\,[x_i-2^{k-1},x_i + 2^{k-1}]:\, 2^kx\in \mathbb{Z}^d
	\}$ denote the set of dyadic cubes in $\mathbb{R}^d$ with side-length $2^{-k}$ centered at points in the dyadic lattice $2^k\mathbb{Z}^d$. For any $\Omega\subseteq \mathbb{R}^d$, we write $\mathbb{D}_k^{\Omega}\eqdef \{Q\in \mathbb{D}_k:\, Q\cap \Omega \neq \emptyset\}$. We consider a
spline-based MRA based on~\eqref{eq:splines_tensor} over the cube $[0,1]^d$ is constructed as follows: for every $k\in \mathbb{N}_+$ and each dyadic cube $Q\eqdef Q_{\mathbf{j:k}}\in \mathbb{D}_k$
we associate a spline $\mathcal{N}_{\mathbf{j},k:I}:\mathbb{R}\to\mathbb{R}$ defined by
\begin{equation}
	\label{eq:MRA_Spline}
	\mathcal{N}_{\mathbf{j},k:I}(x)
	\eqdef
	\mathcal{N}_I(2^kx - \mathbf{j})
	.
\end{equation}
Our next lemma shows that the Res-KAN networks can \textit{exactly} emulate the basic wavelet-type splines in~\eqref{eq:MRA_Spline}.

\begin{lemma}[Multi-Dimensional Cardinal $B$-Spline Implementation]
	\label{lem:UniverariateSplines}
	Let $I,d,k\in \mathbb{N}_+$, and $\mathbf{j}\in 2^{-k}\mathbb{Z}^k$.
	There exists a Res-KAN network $\hat{\mathcal{N}}_{d:I}:\mathbb{R}^d\to \mathbb{R}$ such that, for each $x\in \mathbb{R}^d$
	\allowdisplaybreaks
	\begin{align*}
		\hat{\mathcal{N}}_{k,\mathbf{j}:I}(x)
		=
		\mathcal{N}_{k,\mathbf{j}:I}(x)
	\end{align*}
	Moreover, $\hat{\mathcal{N}}_{k,\mathbf{j}:I}$
	has width at most $d+4$, depth $2d-1$, and $d^2+35d-36$ non-zero parameters.
	Furthermore, only the $2d$ non-zero parameters in the layer depend on $k$ and on $\mathbf{j}$.
\end{lemma}
\begin{proof}[{Proof of Lemma~\ref{lem:UniverariateSplines}}]
	We consider the case where $d>1$ with the case where $d=1$ following the definition of the Res-KAN. Consider the Res-KAN $\hat{\mathcal{N}}_{d:I}:\mathbb{R}^d\to \mathbb{R}$ given by for each $x\in \mathbb{R}^d$
	\[
		\hat{\mathcal{N}}_{d:I}(x)
		\eqdef
		\times^2_d\Big(
		\sigma_{e_{I+1}}\bullet I_d x+ \mathbf{0}_d
		\Big)
	\]
	where $e_1,\dots,e_{I+1}$ are the standard basis vectors of $\mathbb{R}^{I+1}$, $\bullet$ denotes componentwise composition, and $\times_d^2$ is the same as in Lemma~\ref{lem:MultNet}.
	By Lemma~\ref{lem:MultNet} $\hat{\mathcal{N}}_{d:I}$ has width at most $d+4$, depth $2(d-1)$, and $d^2+33d-36$ non-zero parameters; moreover, we have
	\allowdisplaybreaks
	\begin{align*}
		\hat{\mathcal{N}}_{d:I}(x)
		 & =
		\prod_{i=1}^d
		\Big(
		\sigma_{e_{I+1}}\bullet I_d x+ \mathbf{0}_d
		\Big)
		\\
		 & =
		\prod_{i=1}^d
		\Big(
		\mathcal{N}_{I}(x_k)
		\Big)
		\\
		 & =
		\mathcal{N}_{d:I}(x)
	\end{align*}
	where the last equality holds by the definition of $\mathcal{N}_{d:I}$ in~\eqref{eq:splines_tensor}. Consequently, for each $k\in \mathbb{N}_+$ and each $\mathbf{j}\in 2^{-k}\,\mathbb{Z}^k$ we have
	\allowdisplaybreaks
	\begin{align*}
		\hat{\mathcal{N}}_{\mathbf{j},k:I}(x)
		\eqdef
		\mathcal{N}_{d:I}(2^kx-\mathbf{j})
		=
		\hat{\mathcal{N}}_{d:I}(2^kx-\mathbf{j})
		.
	\end{align*}
	Since the map $x\mapsto 2^kx-\mathbf{j}$ is an affine map from $\mathbb{R}^d$ to itself, then $
		\hat{\mathcal{N}}_{k,\mathbf{j}:I}(\cdot)\eqdef \hat{\mathcal{N}}_{d:I}(2^k \cdot-\mathbf{j})$ also has depth $2d-1$, width at most $d+4$, and a simple count shows that it has $d^2+35d-36$ non-zero parameters.
\end{proof}
Using one of the main results of~\cite{DeVoreSharpleyBesovDomains1993}, we are now able to describe the Besov spaces $B_{p,q}^{\alpha}([0,1]^d)$ in terms of the Res-KAN networks. We reiterate that the key point here is the approximability of the Besov norm.
\begin{lemma}[Approximation in Besov Norm by Normalized Residual KAN Network]
	\label{lem:BesovRep_Cube}
	Let $0 < q,p \le \infty$ and $\alpha<s < \infty$,
	for each $f \in B^{s}_{p,q}([0,1]^d)$. For every ``simultaneous approximation error'' $\varepsilon>0$, there exists a Res-KAN $\hat{f}_{\varepsilon}:\mathbb{R}^d\to \mathbb{R}$ satisfying
	\[
		\big\|
		f
		-
		\hat{f}_{\varepsilon}
		\big\|_{B^{\alpha}_{p,q}([0,1]^d)}
		<
		\varepsilon
		.
	\]
	Moreover, $\hat{f}$ has width at most $
		(d+4)\big(2^{K+1} - 2\big)
	$, depth at most $2d-1$, and no more than $
		\left(d^2 + 35d - 35\right)(2^{K+1} - 1)
	$ non-zero parameters; where $K\le \big\lceil \log_2 \varepsilon^{1/(\alpha - s)} - 1 \big\rceil$.
\end{lemma}
\begin{proof}[{Proof of Lemma~\ref{lem:BesovRep_Cube}}]
	\hfill\\
	\noindent\textbf{Step $1$ - Asymptotic Spline Expansion:}
	Set $r\eqdef \lceil \alpha \rceil$, $\lambda \eqdef \min\{r, r - 1 + \frac{1}{p}\}$, $0 < q, p <\infty$, and $0 < \alpha < \lambda$.
	Set $I\eqdef r$.
	By~\cite[Corollary 5.3]{DeVorePopov_1988InerpolationofBesovSpaces} we know that any $f \in L_p([0,1]^d)$ belongs to $B_{p,q}^s([0,1]^d)$ if and only if
	there is a real sequence $\beta_{\cdot}^f\eqdef (\beta_{Q:k}^f)_{Q\in \mathbb{D}_k^{[0,1]^d},k\in\mathbb{N}_+}$
	such that
	\begin{equation}
		\label{eq:SplineTime}
		f
		=
		\sum_{k\in \mathbb{N}}
		\sum_{Q_{\mathbf{j}} \in \mathbb{D}_k^{[0,1]^d}}
		\beta_{Q:k}^f
		\mathcal{N}_{\mathbf{j},k:r}
	\end{equation}
	for all $x\in [0,1]^d$ with coefficient sequence $\beta_{\cdot}^f$ satisfying
	\begin{equation}
		\label{eq:sequential_besov_space_identification}
		c\| f\|_{B_{q,p}^{s}([0,1]^d)}^q
		\le
		\|f\|_{\operatorname{spline}:\alpha,p,q}
		\le
		C\| f \|_{B_{q,p}^{s}([0,1]^d)}^1
	\end{equation}
	for some absolute constants $0<c\le C<$ depending only on $p,q,s$, and on $d$; where the quasi-norm
	$\|\cdot\|_{\operatorname{spline}:\alpha,p,q}$ is given by
	\begin{equation}
		\label{eq:spline_quasiNorm__Equivalent}
		\|f\|_{\operatorname{spline}:\alpha,p,q}^q
		\eqdef
		\sum_{k=0}^{\infty} 2^{\alpha k q}
		\,
		\Biggl(
		\sum_{j \in \Lambda(k)} |\beta_{\mathbf{j},k}|^p 2^{-k d}
		\Biggr)^{q/p}
	\end{equation}
	and where $\Lambda(k)\eqdef \{ \mathbf{j}\in 2^{-k}\mathbb{Z}^d\,:\, \beta_{\mathbf{j}} \neq 0\}$ (see~\cite[page 402 above Equation (4.8)]{DeVorePopov_1988InerpolationofBesovSpaces}).

	Observe that, again by~\cite[Corollary 5.3]{DeVorePopov_1988InerpolationofBesovSpaces}, since $f\in B_{p,q}^{s}([0,1]^d)$
	then $\|f\|_{\operatorname{spline}:s,p,q}<\infty$. Thus,~\eqref{eq:spline_quasiNorm__Equivalent} implies that: there is some $C>0$ such that
	\begin{equation}
		\label{eq:convergence_series_bound}
		\sup_{k\in \mathbb{N}}\,
		2^{s k q}
		\Biggl(
		\sum_{j \in \Lambda(k)} |\beta_{\mathbf{j},k}|^p 2^{-k d}
		\Biggr)^{q/p}
		<
		C
		.
	\end{equation}
	Thus,~\eqref{eq:convergence_series_bound} implies the following growth condition on the terms $\Biggl(
		\sum_{j \in \Lambda(k)} |\beta_{\mathbf{j},k}|^p 2^{-k d}
		\Biggr)^{q/p} $
	\begin{equation}
		\label{eq:convergence_series_GROWTH__bound}
		\Biggl(
		\sum_{j \in \Lambda(k)} |\beta_{\mathbf{j},k}|^p 2^{-k d}
		\Biggr)^{q/p}
		<
		C
		\,
		2^{-s k q}
		.
	\end{equation}

	\noindent \textbf{Step $2$ - Spline Emulation:}
	Applying Lemma~\ref{lem:UniverariateSplines}, once for each $k\in \mathbb{N}_+$ and every $\mathbf{j}\in \Lambda(k)$, we deduce the existence of a sequence of Res-KAN networks $
		\big\{
		\hat{\mathcal{N}}_{k,\mathbf{j}:I}
		\big\}_{k\in \mathbb{N}_+,\, \mathbf{j}\in \Lambda(j)}
	$ satisfying the global emulation property
	\begin{align*}
		\hat{\mathcal{N}}_{k,\mathbf{j}:I}(x)
		=
		\mathcal{N}_{k,\mathbf{j}:I}(x)
	\end{align*}
	for each $x\in \mathbb{R}^d$. Again, each such $\hat{\mathcal{N}}_{k,\mathbf{j}:I}$ has width at most $d+4$, depth $2d-1$, and $d^2+35d-36$ non-zero parameters, and only the first layer depends on $k$ and on $\mathbf{j}$, with all other parameters being shared. Thus, the asymptotic expansion in~\eqref{eq:SplineTime} may be re-expressed purely as a convergent series of Res-KANs
	\begin{equation}
		\label{eq:SplineTime__PREResKANs}
		f
		=
		\sum_{k\in \mathbb{N}}
		\sum_{\mathbf{j}\in \Lambda(k)}
		\beta_{Q:k}^f
		\hat{\mathcal{N}}_{\mathbf{j},k:I}
	\end{equation}
	again, with the (actually the same) coefficient sequence $\beta_{\cdot}^d$ inducing a finite quasi-norm~\eqref{eq:spline_quasiNorm__Equivalent}.

	\noindent \textbf{Step $3$ - Finite-Parameterized Truncation:}
	We now construct our approximator by truncating the expansion in~\eqref{eq:SplineTime__PREResKANs} as follows. For each $K\in \mathbb{N}_+$, to be set retroactively, define
	\begin{equation}
		\label{eq:SplineTime__ResKANs}
		\hat{f}_K
		\eqdef
		\sum_{k=0}^K
		\,
		\sum_{\mathbf{j}\in \Lambda(k)}
		\beta_{Q:k}^f
		\hat{\mathcal{N}}_{\mathbf{j},k:I}
		.
	\end{equation}
	It remains the bound on the Besov norm of $\|f-\hat{f}_K\|_{B^{\alpha}_{p,q}([0,1]^d)}$ above. For this, using~\eqref{eq:SplineTime__PREResKANs},~\eqref{eq:SplineTime__ResKANs}, and again relying on~\eqref{lem:UniverariateSplines} we have that
	\allowdisplaybreaks
	\begin{align*}
		\|
		f-\hat{f}_K
		\|_{B^{\alpha}_{p,q}([0,1]^d)}
		 & \lesssim
		\biggl\|
		\sum_{k\in \mathbb{N}}
		\sum_{\mathbf{j}\in \Lambda(k)}
		\beta_{Q:k}^f
		\mathcal{N}_{\mathbf{j},k:r}
		-
		\sum_{k=0}^K\,
		\sum_{\mathbf{j}\in \Lambda(k)}
		\beta_{Q:k}^f
		\hat{\mathcal{N}}_{\mathbf{j},k:r}
		\biggr\|_{B^{\alpha}_{p,q}([0,1]^d)}
		\\
		 & =
		\biggl\|
		\sum_{k\in \mathbb{N}}
		\sum_{\mathbf{j}\in \Lambda(k)}
		\beta_{Q:k}^f
		\mathcal{N}_{\mathbf{j},k:r}
		-
		\sum_{k=0}^K\,
		\sum_{Q_{\mathbf{j}} \in \mathbb{D}_k^{[0,1]^d}}
		\beta_{Q:k}^f
		\mathcal{N}_{\mathbf{j},k:r}
		\biggr\|_{B^{\alpha}_{p,q}([0,1]^d)}
		\\
		 & =
		\biggl\|
		\sum_{k=K+1}^{\infty}
		\sum_{\mathbf{j}\in \Lambda(k)}
		\beta_{Q:k}^f
		\mathcal{N}_{\mathbf{j},k:r}
		\biggr\|_{B^{\alpha}_{p,q}([0,1]^d)}
		\\
    \numberthis
    \label{eq:Besov_to_besov}
		 & \lesssim
		\left(
            \sum_{k=K+1}^{\infty}
    		2^{\alpha k q}
    		\,
    		\Biggl(
    		\sum_{j \in \Lambda(k)} |\beta_{\mathbf{j},k}|^p 2^{-k d}
    		\Biggr)^{q/p}
        \right)^{{1/q}}
		\\
		 & \lesssim
        \left(
    		\sum_{k=K+1}^{\infty}
    		C
    		2^{\alpha k q}
    		2^{-s k q}
        \right)^{{1/q}}
		\\
		 & = C
        \left(
    		\sum_{k=K+1}^{\infty}
    		\left(2^{(\alpha - s)q}\right)^k
        \right)^{{1/q}}
		\\
		 & = C
        \left(
		      \frac{2^{(\alpha - s)q (K+1)}}{1 - 2^{q(\alpha - s)}}
        \right)^{{1/q}}
		\\
		\numberthis
		\label{eq:estimated}
		 & = C_{\alpha,s,q}
		\,
		2^{(\alpha - s)(K+1)}
	\end{align*}
	where $C_{\alpha,s,q}\eqdef C/(1-2^{q(\alpha-s)})^{{1/q}}>0$
    {and~\eqref{eq:Besov_to_besov} holds by the decay rates of the coefficients $\beta_{\cdot}^f$ implied by the inclusion of $f\in B_{p,q}^{\alpha}([0,1]^d)$; see~\cite[Theorem 4.22 (ii) and Equation (4.109)]{Triebel_FSBookIII_2006}.}
    .  Thus, for any given $\varepsilon>0$, we retroactively set
	\begin{equation}
		\label{eq:K_epsilon_relationship}
		K
		=
		\biggl\lceil
		\frac{\log_2(\varepsilon)}{(\alpha - s)
        } - 1
		\biggr\rceil
		=
		\Big\lceil
		\log_2 \varepsilon^{1/((\alpha - s)
        )} - 1
		\Big\rceil
	\end{equation}
	and re-label $\hat{f}_{\varepsilon}\eqdef \hat{f}_{K}$ then,~\eqref{eq:estimated} implies that
	\[
		\|
		f-\hat{f}_{\varepsilon}
		\|_{B^{\alpha}_{p,q}([0,1]^d)}
		\lesssim_{\alpha,s,q}
		\varepsilon
	\]
	where $\lesssim_{\alpha,s,q}$ is used to a constant only depending on $\alpha$, $s$, and on $q$.

	\noindent \textbf{Step $4$ - Verifying The Neural Representation:}
	Since, for each $k\in [K]$ and every $j\in \Lambda(k)$, the networks $\hat{\mathcal{N}}_{\mathbf{j},k:I}$ have the same depth, then we may represent~\eqref{eq:SplineTime__ResKANs} by a neural network of width
	\[
		\sum_{k=0}^K \,
		\sum_{\mathbf{j}\in \Lambda(k)} (d+4)
		\le
		(d+4)\sum_{k=1}^K 2^k
		=
		(d+4)\big(2^{K+1} - 2\big)
	\] depth $2d-1$, and with at most
	\[
		\sum_{k=0}^K \# \sum_{\mathbf{j}\in \Lambda(k)} \biggl(
		(d^2+35d-36) +1\biggr)
		\le
		\sum_{k=0}^K 2^k\left(d^2+35d-35\right) = \left(d^2+35d-35\right)(2^{K+1} - 1)
	\]
	non-zero parameters. Namely via the representation
	\[
		\hat{f}_{\varepsilon}(x)
		\eqdef
		\Biggl(
		\bigoplus_{k\in [K],j\in \Lambda(k)} \beta_{Q:k}^f
		\Biggr)^{\top}
		\Biggl(
		\bigoplus_{k\in [K],j\in \Lambda(k)} \hat{\mathcal{N}}_{\mathbf{j},k:I}
		\big(
			\mathbf{1}_{\sum_{k\in [K]\sum_{j\in \Lambda(k)}}}
			\,
			x
			\big)
		\Biggr)
		.
	\]
	This completes our proof.
\end{proof}
We remind the reader that this directly implies Sobolev-norm approximation guarantees since $B_{p,p}^s(\Omega)$ is equivalent to the Sobolev space $W^{s,p}(\Omega)$ when $s$ is not an integer; see \eg,~\cite[Equation and Remark 5.31]{Triebel_FSBookIII_2006}.
We now extend the conclusion of Lemma~\ref{lem:BesovRep_Cube} to general domains.

\subsection{From the Unit Cube to General Domains}

We now extend Lemma~\ref{lem:BesovRep_Cube} to general domains. We consider two cases: classical regular domains or domains of fractal type. Both are treated separately, and their joint conclusion is the main result of this section.

\subsubsection{Regular Domains}

The following result is true on $(\varepsilon,\delta)$-domains in $\mathbb{R}^d$; {remark that every Lipschitz domain is an $(\varepsilon,\delta)$-domain.}
We may now extend our results from the $d$-dimensional unit hypercube to general Lipschitz domains by using Rychkov's extension operator, introduced in~\cite{Rychkov_BesovTriebelLizorkinExtensions_1999}, as improved on in~\cite{Shi2024new__Rychkov_BesovTriebelLizorkinExtensions} or the extension operator used in~\cite{rogers2006degree}.
\begin{lemma}[Approximation in Besov Norm by Normalized Residual KAN Network - $(\varepsilon,\delta)$-Domain]
	\label{lem:BesovRep_Cube__eps_delt_case}
	Let $0 < q,p,\alpha<s < \infty$, let $\Omega\subseteq [0,1]^d$ be an $(\epsilon,\delta)$-domain.
	For each $f \in B^{s}_{p,q}([0,1]^d)$, every ``simultaneous approximation error'' $\varepsilon>0$, there exists a Res-KAN $\hat{f}_{\varepsilon}:\mathbb{R}^d\to \mathbb{R}$ satisfying
	\[
		\big\|
		f
		-
		\hat{f}_{\varepsilon}
		\big\|_{B^{\alpha}_{p,q}(\Omega)}
		<
		\varepsilon
		.
	\]
	Moreover, $\hat{f}$ has width at most $
		(d+4)\big(2^{K+1} - 2\big)
	$, depth at most $2d-1$, and no more than $
		\left(d^2 + 35d - 35\right)(2^{K+1} - 1)
	$ non-zero parameters; where $K\le \big\lceil \log_2 \varepsilon^{1/(\alpha - s)} - 1 \big\rceil$.
\end{lemma}
\begin{proof}
	As remarked at the start of~\cite[Chapter 2]{Rychkov_BesovTriebelLizorkinExtensions_1999} any bounded extension operator on special Lipschitz domains implies a bounded extension operator on general $(\varepsilon,\delta)$-domains (by gluing using a partition of unity); thus, the main results of~\cite{rogers2006degree} 
    implies that there exists a bounded linear operator $\mathcal{E}:B_{p,q}^{\alpha}(\Omega)\to
		\mathcal{E}:B_{p,q}^{\alpha}(\mathbb{R}^d)$. Since the restriction operators between $(\epsilon,\delta)$-domains are bounded
	then we have the following estimate: for any $f\in B_{p,q}^{\alpha}(\Omega)$ and any Res-KAN $\hat{f}$
	\allowdisplaybreaks
	\begin{align*}
		\|f-\hat{f}|_{\Omega}\|_{B_{p,q}^{\alpha}(\Omega)}
		 & \lesssim
		\|\mathcal{E}(f)-\mathcal{E}(\hat{f})|_{\Omega}\|_{B_{p,q}^{\alpha}(\mathbb{R}^d)}
		\\
		 & \lesssim
		\|\mathcal{E}(f)|_{[0,1]^d}-\mathcal{E}(\hat{f})|_{\Omega}|_{[0,1]^d}\|_{B_{p,q}^{\alpha}([0,1]^d)}
		\\
		 & =
		\|\mathcal{E}(f)|_{[0,1]^d}-\mathcal{E}(\hat{f})|_{[0,1]^d}\|_{B_{p,q}^{\alpha}([0,1]^d)}
		\\
		\numberthis
		\label{eq:ready_to_apply_cube_approx}
		 & =
		\|
		\mathcal{E}(f)|_{[0,1]^d}
		-
		\hat{f}|_{[0,1]^d}
		\|_{B_{p,q}^{\alpha}([0,1]^d)}
	\end{align*}
	where~\eqref{eq:ready_to_apply_cube_approx} is followed by the definition of the restriction extension operators.
	Since $\mathcal{E}(f)|_{[0,1]^d}\in B_{p,q}^{\alpha}([0,1]^d)$, then we may retroactively pick our Res-KAN $\hat{f}$ as in Proposition~\eqref{lem:BesovRep_Cube} to bound the right-hand side above by $\varepsilon$; \ie
	\allowdisplaybreaks
	\begin{align*}
		\|f-\hat{f}|_{\Omega}\|_{B_{p,q}^{\alpha}(\Omega)}
		\lesssim
		\|
		\mathcal{E}(f)|_{[0,1]^d}
		-
		\hat{f}|_{[0,1]^d}
		\|_{B_{p,q}^{\alpha}([0,1]^d)}
		<
		\varepsilon
		.
	\end{align*}
	This completes our proof.
\end{proof}

\subsubsection{Fractal Domains}

\begin{lemma}[KAN-Approximation of Besov Functions on Fractal Domains]
	\label{lem:BesovApprox_on_Fractal}
	Let $0<\alpha<s<\infty$ and $1\le p,q<\infty$, $d-1<n< d$ and $\mathcal{X}\subseteq [0,1]^d$ be an Ahlfors $n$-regular. Then, for every $f\in B^{s}_{p,q}(\mathcal{X})$ and every $\varepsilon>0$ there is a Res-KAN $\hat{f}$ such that: $\hat{f}|_{\mathcal{X}}$ is indeed a well-defined element of $B_{p,q}^{\alpha-(n-d)/p}(\mathcal{X})$ and satisfies
	\[
		\|f-\hat{f}|_{\mathcal{X}}\|_{B^{\alpha - (n-d)/p}_{p,q}(\mathcal{X})}
		<
		\varepsilon
		.
	\]
	Moreover, $\hat{f}$ has width at most $
		(d+4)\big(2^{K+1} - 2\big)
	$, depth at most $2d-1$, and no more than $
		\left(d^2 + 35d - 35\right)(2^{K+1} - 1)
	$ non-zero parameters; where $K\le \big\lceil \log_2 \varepsilon^{1/(\alpha - (n-d)/p - s)} - 1 \big\rceil$.

\end{lemma}
\begin{proof}[{Proof of Lemma~\ref{lem:BesovApprox_on_Fractal}}]
	By the Whitney-type extension result in~\cite[Theorem 6.1]{Ihnatseyva_AhlforsRegularJFA_2013}, there exists an $\mathcal{E}:B^{s - (n-d)/p}_{p,q}(\mathcal{X})
		\to
		B^{s}_{p,q}(\mathbb{R}^d)$ and a constant $c_{n,d,p,\mathcal{X}}>0$ such that
	\begin{equation}
		\label{eq:norm_bound}
		\|\mathcal{E}(f)\|_{B^{s}_{p,q}(\mathcal{X})}
		\le
		c_{n,d,p,\mathcal{X}}
		\,
		\|f\|_{B^{s - (n-d)/p}(\mathbb{R}^d)}
	\end{equation}
	for each $f\in B^{s - (n-d)/p}_{p,q}(\mathcal{X})$. Moreover, $\mathcal{E}$ is a left-inverse of the restriction map. Therefore, for each $f\in B^{s}_{p,q}(\mathcal{X})$ and every $\varepsilon>0$, Lemma~\ref{lem:BesovRep_Cube} guarantees that there is a Res-KAN $\hat{f}$ satisfying
	\begin{equation}
		\label{eq:extended_approximation}
		\|\mathcal{E}(f)|_{[0,1]^d}-\hat{f}|_{[0,1]^d}\|_{B^{\alpha}_{p,q}([0,1]^d)}
		<
		\varepsilon
		.
	\end{equation}
	Since the restriction operators from Besov spaces on $\mathbb{R}^d$ to $(\varepsilon,\delta)$-domains are bounded linear operators; then
	there is a constant $C_{p,q,\alpha,d}>0$ such that: for every $g\in B^{\alpha}_{p,q}([0,1]^d)$, in particular for $g=\mathcal{E}(f)-\hat{f}$, we have
	\begin{equation}
		\label{eq:extension_again}
		\|g\|_{B^{\alpha}_{p,q}(\mathbb{R}^d)}
		\le
		C_{p,q,\alpha,d}
		\|g|_{[0,1]^d}\|_{B^{\alpha}_{p,q}([0,1]^d)}
		.
	\end{equation}
	By the restriction theorem in~\cite[Theorem 4.1.]{Ihnatseyva_AhlforsRegularJFA_2013}, since $\hat{f}\in B_{p,q}^{\alpha}(\mathbb{R}^d)$ then its restriction $\hat{f}|_{\mathcal{X}}$ is a well-defined element of $B^{\alpha - (n-d)/p}(\mathbb{R}^d)$.
	We may therefore, comfortably combine~\eqref{eq:norm_bound} and~\eqref{eq:extension_again} with~\eqref{eq:extended_approximation} to obtain
	\allowdisplaybreaks
	\begin{align*}
		\|f-\hat{f}|_{\mathcal{X}}\|_{B^{\alpha - (n-d)/p}(\mathbb{R}^d)}
		 & \le
		c_{n,d,p,\mathcal{X}}
		\,
		\|\mathcal{E}(f)-\hat{f}\|_{B^{\alpha - (n-d)/p}(\mathbb{R}^d)}
		\\
		 & \le
		c_{n,d,p,\mathcal{X}}
		C_{p,q,\alpha,d}
		\,
		\|\mathcal{E}(f)|_{[0,1]^d}-\hat{f}|_{[0,1]^d}\|_{B^{\alpha}_{p,q}([0,1]^d)}
		\\
		 &
		\le
		c_{n,d,p,\mathcal{X}}
		C_{p,q,\alpha,d}
		\,
		\varepsilon
		.
	\end{align*}
	Relabeling the constant $\varepsilon>0$ as $\varepsilon/(c_{n,d,p,\mathcal{X}} C_{p,q,\alpha,d})$ yields the conclusion.
\end{proof}
We now obtain our main approximation theorem.
\begin{proof}[{Proof of Theorem~\ref{thrm:Main_Approximation}}]
	Together Lemmata~\ref{lem:BesovRep_Cube__eps_delt_case} and~\ref{lem:BesovApprox_on_Fractal} now imply Theorem~\ref{thrm:Main_Approximation}.
\end{proof}

\section{{Proof of Generalization Bound}}
\label{s:Proof_Generalization}

Before proving our main results, we recall some technical definitions from learning theory on which our proof rests.
\subsection{Definitions for our Learning Theoretic Guarantees}
Our analysis in this section relies on the following dimensions from classical learning theory.
\begin{definition}[Growth function, VC-dimension, Shattering]
	Let $\mathcal{H}$ denote a class of functions from $\mathcal{X}$ to $\{0,1\}$ (the hypotheses, or the classification rules). For any non-negative integer $m$, we define the growth function of $\mathcal{H}$ as
	$$
		\Pi_{\mathcal{H}}(m) \eqdef \max_{x_1, \ldots, x_m \in \mathcal{X}} \left| \{ (h(x_1), \ldots, h(x_m)) : h \in \mathcal{H} \} \right|.
	$$
	If $\left| \{ (h(x_1), \ldots, h(x_m)) : h \in \mathcal{H} \} \right| = 2^m$, we say $\mathcal{H}$ shatters the set $\{x_1, \ldots, x_m\}$. The Vapnik-Chervonenkis dimension of $\mathcal{H}$, denoted $\operatorname{VCdim}(\mathcal{H})$, is the size of the largest shattered set, \ie, the largest $m$ such that $\Pi_{\mathcal{H}}(m) = 2^m$. If there is no largest $m$, we define $\operatorname{VCdim}(\mathcal{H}) = \infty$.
\end{definition}
\begin{definition}[Pseudodimension]
	Let $\mathcal{H}$ be a class of functions from $\mathcal{X}$ to $\mathbb{R}$. The pseudodimension of $\mathcal{H}$, written $\operatorname{Pdim}(\mathcal{H})$, is the largest integer $m$ for which there is an
	$(x_i)_{i=1}^m \oplus (y_i)_{i=1}^m \in \mathcal{X}^m \times \mathbb{R}^m$ satisfying: for any $(b_1, \ldots, b_m) \in \{0,1\}^m$ there exists $h \in \mathcal{H}$ such that
	$$
		\forall i : h(x_i) > y_i \iff b_i = 1
	$$
\end{definition}
The pseudodimension is not scale sensitive. This is not the case for the $\gamma$-fat shattering dimension defined as follows.
\begin{definition}[$\gamma$-Fat Shattering and $\gamma$-Fat Shattering Dimension]
	Let $\mathcal{H} \subseteq [0,1]^{\mathcal{X}}$. We say that $\mathcal{H}$ $P_{\gamma}$-shatters a set $X = \{x_1, \ldots, x_n\} \subseteq \mathcal{X}$ if there exist $s_1, \ldots, s_n$ such that, for all $E \subseteq X$, there is an $h \in \mathcal{H}$ satisfying
	\[
		\forall x_i \in E, \quad h(x_i) \geq s_i + \gamma
	\]
	\[
		\forall x_i \in X - E, \quad h(x_i) \leq s_i - \gamma
	\]
	\textit{The fat shattering dimension of $\mathcal{H}$ at scale $\gamma$, denoted by $\operatorname{Pdim}_{\gamma}(\mathcal{H})$, is the size of a largest $P_{\gamma}$-shattered set.}
\end{definition}
We may now prove our main statistical guarantee.

\subsection{The Proof}

We now prove our learning guarantees, beginning with Lemma~\ref{lem:PseudoDimBound}.
Thus, we must first recall the following result relating the pseudodimension of a set of binary classifiers implemented by the neural network to the VC-dimension of a modification of that class with two extra computational units. We emphasize that the following result holds for general feedforward neural networks (\ie, neural networks given by a connected directed acyclic graph on which a computation is executed on every node, which is neither initial (input node) nor terminal (output node)); see \eg,~\cite{karpinski1997polynomial} for a clean formulation.

\begin{lemma}[Fat-Shattering Dimension Bound for KANs]
	\label{lem:PseudoDimBound}
	Let $\alpha>0$, $d,L,W,I\in \mathbb{N}_+$, with $\lceil\alpha \rceil \le I$. Then, for every $\gamma>0$
	\[
		\operatorname{Pdim}_{\gamma}\big(\operatorname{Res-KAN}_{L,W}^{I,\alpha}(\mathbb{R}^d,\mathbb{R})\big)
		\in
		\mathcal{O}\left(
		L^2\,W^2\,(I+3-\lceil \alpha\rceil)
		\big(
		\log\big(
		L\,W^2\,(I+3-\lceil \alpha\rceil)
		\big) + L
		\big)
		\right)
		.
	\]
\end{lemma}

\begin{proof}
	First, the smoothness condition in~\eqref{eq:DEF_SRKs__sparsity} implies that each KAN-neuron in~\eqref{eq:KAN_r} can be represented as a feedforward network with computational graph
	\[
		\hat{G}=(\{z_0,z_2\}\cup \{z_{1:i}\}_{i=0}^{I+3-\lceil \alpha\rceil }, \{\{z_0,z_{1:i}\}_{i=0}^{I+3-\lceil \alpha\rceil}\cup \{z_{1:i},z_2\}_{i=0}^{I+3-\lceil \alpha\rceil}\}
	\]
	with input node $z_0$, output node $z_2$, and for each computational node $z_{1:0},\dots,z_{1:I+3-\lceil \alpha\rceil }$ we have
	\[
		z_{1:i} = \beta_i p_i(x)
	\]
	where $p_i\eqdef \mathcal{N}_i$ (viewed as a piecewise polynomial of degree $I+1$) with, of course, no more than $I+1$ breakpoints.

	Each fully-connected Res-KAN layer, as defined in~\eqref{eq:norm_res_KAN_layers}, can be represented as a feedforward neural network with $2$ layers (including the input and output layers), and at most $2d_{out}(d_{in}+1)(2(I+3-\lceil \alpha\rceil)$ non-zero parameters, and $(2(I+3-\lceil \alpha\rceil)d_{out}$ computational units. Consequently, every $\hat{f}\in \operatorname{Res-KAN}_{L,W}^{I,\alpha}(\mathbb{R}^d,\mathbb{R})$ can be represented as a feedforward neural network at most $W^{\prime}\eqdef L2d_{out}(d_{in}+1)(2(I+3-\lceil \alpha\rceil)$ non-zero parameters and at most $L(2(I+3-\lceil \alpha\rceil)d_{out}$ computational units, arranged into at most $L^{\prime}\eqdef 2L$ layers.

	Following~\cite[Theorem 14.1]{anthony2009neuralBook}, for every $\hat{f}\in \operatorname{Res-KAN}_{L,W}^{I,\alpha}(\mathbb{R}^d,\mathbb{R})$ let $\tilde{f}:\mathbb{R}^d\to \{0,1\}$ defined by modifying the computational graph of $\hat{f}$ as follows:
	We added one extra input unit and one extra computation unit. This additional computation unit is a $I_{[0,\infty)}$ (heavyside activation function) unit receiving input only from the output unit of $\hat{f}$ and from the new input unit, and it is the output unit of $\tilde{f}$. Let $\mathcal{F}$ consist of all functions constructed in this manner by modifying some $\hat{f}\in \operatorname{Res-KAN}_{L,W}^{I,\alpha}(\mathbb{R}^d,\mathbb{R})$ in this way.

	Thus,~\cite[Theorem 2.1]{bartlett1998almost} implies that the VC-dimension of $\operatorname{Res-KAN}_{L,W}^{I,\alpha}(\mathbb{R}^d,\mathbb{R})$ is at most
	\begin{equation}
		\label{eq:Bartlett_VC_PWPoly}
		\operatorname{VCdim}(
		\mathcal{F}
		)
		\in
		\mathcal{O}(W^{\prime}L^{\prime} \log W^{\prime} + W^{\prime}(L^{\prime})^2)
		.
	\end{equation}
	Now~\cite[Theorem 14.1]{anthony2009neuralBook} implies that the pseudodimension $\operatorname{Pdim}(\operatorname{Res-KAN}_{L,W}^{I,\alpha}(\mathbb{R}^d,\mathbb{R}))$ satisfies
	\begin{equation}
		\label{eq:PDim_VC}
		\operatorname{Pdim}\big(\operatorname{Res-KAN}_{L,W}^{I,\alpha}(\mathbb{R}^d,\mathbb{R})\big)
		\lesssim
		\operatorname{VCdim}(\mathcal{F})
		.
	\end{equation}
	Combining~\eqref{eq:Bartlett_VC_PWPoly} with~\eqref{eq:PDim_VC} yields
	\begin{equation}
		\label{eq:pseudo__bound}
		\begin{aligned}
			\operatorname{Pdim}\big(\operatorname{Res-KAN}_{L,W}^{I,\alpha}(\mathbb{R}^d,\mathbb{R})\big)
			 & \in
			\mathcal{O}(W^{\prime}L^{\prime} \log W^{\prime} + W^{\prime}(L^{\prime})^2)
			\\
			 & \in
			\mathcal{O}\left(
			L^2\,W^2\,(I+3-\lceil \alpha\rceil)
			\big(
			\log\big(
			L\,W^2\,(I+3-\lceil \alpha\rceil)
			\big) + L
			\big)
			\right)
			.
		\end{aligned}
	\end{equation}
	Now, by~\cite[Theorem 11.13 (i)]{anthony2009neuralBook}, for every $\gamma>0$ we have
	\begin{equation}
		\label{eq:fat_pseudo}
		\operatorname{Pdim}_{\gamma}\big(\operatorname{Res-KAN}_{L,W}^{I,\alpha}(\mathbb{R}^d,\mathbb{R})\big)
		\le
		\operatorname{Pdim}\big(\operatorname{Res-KAN}_{L,W}^{I,\alpha}(\mathbb{R}^d,\mathbb{R})\big)
		.
	\end{equation}
	We obtain our conclusion upon combining~\eqref{eq:pseudo__bound} with~\eqref{eq:fat_pseudo}.
\end{proof}

Consider the hypothesis class $ \mathcal{H}_{W,L}^{\alpha,I,p,q}(\mathcal{X})$ consisting of all maps $h_{\hat{f},f}:\mathcal{X}
{\times \mathbb{R}}
\to \mathbb{R}$ for which there exist some $\hat{f}\in \operatorname{Res-KAN}_{L,W}^{I,\alpha}(\mathbb{R}^d,\mathbb{R})$ and some $f\in B_{p,q}^\alpha(\mathcal{X})$ with $\|f\|_{B_{p,q}^{\alpha}(\mathcal{X})}\le 1$ such that: for every $x\in \mathcal{X}$ we have
\begin{equation}
	\label{eq:difference_class}
	h_{\hat{f},f}(x
        {,y}
    )
	\eqdef
	\big|
	\hat{f}(x)
	-
    	{(}
            f(x)
                {-y}
        {)}
	\big|
	.
\end{equation}
We henceforth denote the unit ball in $B_{p,q}^{\alpha}(\mathcal{X})$ by $B_{p,q}^\alpha(\mathcal{X})_1$. Our next result bounds the fat-shattering dimension of this ``regression error'' hypothesis class $\mathcal{H}_{W,L}^{\alpha,I,p,q}(\mathcal{X})$ in terms of the number of the: degree, regularity parameters $I$ and $\alpha$, as well as the depth and width $L$ and $W$ of our hypothesis class, together with some added constraints on the regularity of the target function in the Besov space which we are learning.


\begin{lemma}[{Fat-Shattering Dimension of the ``Regression-Error'' Class $\mathcal{H}_{W,L}^{\alpha,I,p,q}(\mathcal{X})$}]
	\label{lem:fat_shattering__classH}
	Suppose that $\mathcal{X}$ is a Lipschitz domain, let $1 \le \tau \le \infty$, $1 \le p, q \le \infty$ and $\alpha > \left(
        {(}
            d
        {+1)}
    \left(1/p - 1/\tau\right)\right)_+$ and let $L,I,W\in \mathbb{N}_+$.
	For 
    every $\gamma>0$ we have
	\[
		\operatorname{Pdim}_{\gamma}(\mathcal{H}_{W,L}^{\alpha,I,p,q}(\mathcal{X}))
		\lesssim
		\log_2(1/(8\gamma))^2
		r\, 
		L^2\,W^2
		\big(
		\log\big(
			r
			L\,W^2
			\big) + L
		\big)
		\Big)
		+
		(8\gamma)^{-
            {(}
                d
            {+1)}
        /\alpha}
	\]
	where $r\eqdef I+3-\lceil \alpha\rceil$.
\end{lemma}

\begin{proof}[{Proof of Theorem~\ref{lem:fat_shattering__classH}}]
	Abbreviate $\mathcal{H}\eqdef \mathcal{H}_{W,L}^{\alpha,I,p,q}(\mathcal{X})$.
	For any $A\subseteq C(\mathcal{X})$ and each $\varepsilon>0$ let $\mathcal{N}(\epsilon,A)$ denote the $\epsilon$-covering number of $A$ in $C(\mathcal{X})$ (with the uniform norm). Then, for every $\varepsilon>0$, the definition of $\mathcal{H}^{\prime}$ and the triangle inequality implies that
	\begin{equation}
		\label{eq:doubling_bounds}
		\mathcal{N}(\epsilon,\mathcal{H}^{\prime})
		\le
		\mathcal{N}\big(\epsilon/2,
		\operatorname{Res-KAN}_{L,W}^{I,\alpha}(\mathbb{R}^d,\mathbb{R})
		\big)
		\mathcal{N}\big(\epsilon/2,
		B_{p,q}^\alpha(\mathcal{X}
            {\times \mathbb{R}}
        )_1
		\big)
		.
	\end{equation}
	Taking logarithmic across~\eqref{eq:doubling_bounds} we find that
	\begin{equation}
		\label{eq:doubling_bounds__log}
		\log_2(\mathcal{N}(\epsilon,\mathcal{H}))
		\le
		\log_2\Big(
		\mathcal{N}\big(\epsilon/2,
			\operatorname{Res-KAN}_{L,W}^{I,\alpha}(\mathbb{R}^d,\mathbb{R})
			\big)
		\Big)
		+
		\log_2\Big(
		\mathcal{N}\big(\epsilon/2,
			B_{p,q}^\alpha(\mathcal{X}
                {\times \mathbb{R}}
            )_1
			\big)
		\Big)
		.
	\end{equation}
	By~\cite[Theorem 2.7.4]{van1996weak}, since $1 \le \tau \le \infty$, $1 \le p, q \le \infty$ and $\alpha > \left(
    {(}
        d
    {+1)}
    \left(1/p - 1/\tau\right)\right)_+$ then, we may bound $\log_2\Big(
		\mathcal{N}\big(\epsilon/2,
			B_{p,q}^\alpha(\mathcal{X}
            {\times \mathbb{R}}
            )
			\big)
		\Big)$ above by $\mathcal{O}\big((\varepsilon/2)^{-
         {(}
            d
        {+1)}
        /\alpha}\big)$.
	Consequently,~\eqref{eq:doubling_bounds__log} can be further bounded-above by
	\begin{equation}
		\label{eq:doubling_bounds__log-1}
		\log_2(\mathcal{N}(\epsilon,\mathcal{H}))
		\lesssim
		\log_2\Big(
		\mathcal{N}\big(\epsilon/2,
			\operatorname{Res-KAN}_{L,W}^{I,\alpha}(\mathbb{R}^d,\mathbb{R})
			\big)
		\Big)
		+
		\varepsilon^{-
         {(}
            d
        {+1)}
        /\alpha}
		.
	\end{equation}
	Now, applying~\cite[Theorem 2]{BartlettKulkarniPosner_CoveringFatShattering_1997} we may bound the $\log_2$-covering number of
	\[
		\log_2\Big(
		\mathcal{N}\big(\epsilon/2,
			\operatorname{Res-KAN}_{L,W}^{I,\alpha}(\mathbb{R}^d,\mathbb{R})
			\big)
		\Big)
	\]
	above by its $c_2\varepsilon$-fat shattering dimension; for some absolute constant $c_2>0$ and an additional multiplicative factor of $\log_2(1/\varepsilon)^2$; that is
	\begin{equation}
		\label{eq:order_bound}
		\begin{aligned}
			\log_2\Big(
			\mathcal{N}\big(\epsilon/2,
				\operatorname{Res-KAN}_{L,W}^{I,\alpha}(\mathbb{R}^d,\mathbb{R})
				\big)
			\Big)
			 & \lesssim
			\operatorname{Pdim}_{c_2\,\varepsilon/2}\big(\operatorname{Res-KAN}_{L,W}^{I,\alpha}(\mathbb{R}^d,\mathbb{R})\big)
			\log(2/(c_2\varepsilon))^2
			\\
			 &
			\in
			\mathcal{O}\Big(
			\log_2(1/\varepsilon)^2
			L^2\,W^2\,(I+3-\lceil \alpha\rceil)
			\\
			\qquad\qquad\qquad\qquad
			 &
			\times
			\big(
			\log\big(
			L\,W^2\,(I+3-\lceil \alpha\rceil)
			\big) + L
			\big)
			\Big)
		\end{aligned}
	\end{equation}
	where the order-estimate on the right-hand side of~\eqref{eq:order_bound} follows from Lemma~\ref{lem:PseudoDimBound}. Incorporating~\eqref{eq:order_bound} into the right-hand side of~\eqref{eq:doubling_bounds__log-1} yields
	\begin{equation}
		\label{eq:doubling_bounds__log-2}
		\log_2(\mathcal{N}(\epsilon,\mathcal{H}))
		\lesssim
		\log_2(1/\varepsilon)^2
		L^2\,W^2\,(I+3-\lceil \alpha\rceil)
		\big(
		\log\big(
		L\,W^2\,(I+3-\lceil \alpha\rceil)
		\big) + L
		\big)
		\Big)
		+
		\varepsilon^{-
            {(}
                d
            {+1)}
        /\alpha}
		.
	\end{equation}
	Now, applying~\cite[Theorem 2]{BartlettKulkarniPosner_CoveringFatShattering_1997} again, we have that
	\begin{equation}
		\label{eq:covering_fat_shattering__inequalities}
		\operatorname{Pdim}_{\epsilon/8}(\mathcal {H})
		\le
		\max_{P}\log_2(\mathcal N(\epsilon,\mathcal{H},\mathcal L_1(dP)))
		\lesssim
		\log_2(\mathcal{N}(\epsilon,\mathcal{H}))
	\end{equation}
	where $N(\epsilon,\mathcal{H},\mathcal L_1(dP))$ is the $\epsilon$-covering number of $\mathcal{H}$ with respect to the norm on $\mathbb{E}_{X\sim \mathbb{P}}[\|X\|]$.
	Upon using~\eqref{eq:doubling_bounds__log-2} to bound the right-hand side of~\eqref{eq:covering_fat_shattering__inequalities} and then relabelling $\gamma=\varepsilon/8$, we deduce our conclusion.
\end{proof}

Having bounded the fat-shattering dimension of our ``regression error'' hypothesis class $\mathcal{H}_{W,L}^{\alpha,I,p,q}(\mathcal{X})$, we may obtain the conclusion of Theorem~\ref{thm:Main_Generalization} by appealing to one of the main results of~\cite{AlonBenDavidCesaBianchiHaussler_PDGeneralization}.

\begin{proof}[{Proof of Theorem~\ref{thm:Main_Generalization}}]
	We again abbreviate $\mathcal {H}\eqdef \mathcal{H}_{W,L}^{\alpha,I,p,q}(\mathcal{X})$.
	By~\cite[Theorem 3.6]{AlonBenDavidCesaBianchiHaussler_PDGeneralization}, we have that: for every error size $\varepsilon>0$ and each failure probability $0<\delta \le 1$ the following holds
	\begin{equation}
		\label{eq:AlonBenDavidBound}
		\mathbb{P}\Biggl(
		\sup_{h_{f,\hat{f}}\in \mathcal{H}}
		\,
		\Big|
		\mathbb{E}_{X\sim \mathbb{P}_{\operatorname{smpl}}}(h_{f,\hat{f}}(X))
		-
		\frac1{N}\sum_{n=1}^N\,h_{f,\hat{f}}(X_n)
		\Big|
		\le
		\varepsilon
		\Biggr)
		\ge
		1-\delta
	\end{equation}
	provided that
	\begin{equation}
		\label{eq:PseudoDimensionSize}
		N
		\le
		c_1\left( \frac{1}{\epsilon^2} \left(
			\operatorname{Pdim}_{\epsilon/32}(\mathcal {H})
			\ln^2\left(\frac{
					\operatorname{Pdim}_{\epsilon/32}(\mathcal {H})
				}{\epsilon}\right) + \ln\left(\frac{1}{\delta}\right) \right) \right)
	\end{equation}
	for some absolute constant $c_1>1$. Consequently, if we set the failure probability, $\delta$, to be
	\begin{equation}
		\label{eq:FromSampleSize_to_FailureProbability}
		\delta
		\eqdef
		\exp\Bigl(
		-\tfrac{N\,\epsilon^{2}}{c_{1}}
		+
		\operatorname{Pdim}_{\epsilon/32}(\mathcal{H})
		\,
		\ln^{2}
		\Big(
		\tfrac{\operatorname{Pdim}_{\epsilon/32}(\mathcal{H})}{\epsilon}
		\Big)
		\Bigr)
		.
	\end{equation}
	Note that, we may use the upper-bound for the $\varepsilon/32$-fat shattering dimension computed in Lemma~\ref{lem:fat_shattering__classH}; namely,
	\begin{equation}
		\label{eq:PDim}
		\operatorname{Pdim}_{\epsilon/32}(\mathcal{H})
		\lesssim
		\log_2(4/\epsilon)^2
		r\, 
		L^2\,W^2
		\big(
		\log\big(
			r
			L\,W^2
			\big) + L
		\big)
		\Big)
		+
		(4/\epsilon)^{d/\alpha}
	\end{equation}
	where, as before, $r\eqdef (I+3-\lceil \alpha\rceil)$.
	Therefore such that~\eqref{eq:AlonBenDavidBound} holds if
	$N\le N_{\epsilon,\delta}^{\star}$; where
	\allowdisplaybreaks
	\begin{align}
		N_{\epsilon,\delta}^{\star} & \eqdef\frac{c_1}{\epsilon^2}\Bigl[\,(A+B)\,\ln^2\!\bigl(\tfrac{A+B}{\epsilon}\bigr)
			\;+\;\ln\!\bigl(\tfrac{1}{\delta}\bigr)\Bigr]
		\\
		A                           & \eqdef c\,\log_2\Bigl(\tfrac{4}{\epsilon}\Bigr)^{2}\,r\,L^2\,W^2\bigl(\ln(rLW^2)+L\bigr)\mbox{ and }
		B \eqdef (4/\epsilon)^{
            {(}
                d
            {+1)}
        /\alpha}
	\end{align}
	for some absolute constant $c>0$. Consequently, we have that
	\[
		N
		\in
		\mathcal{O}\Bigl(
		\epsilon^{-2 
        - 
            {(}
                d
            {+1)}
        /\alpha}\,(\ln(1/\varepsilon))^2
		+\epsilon^{-2}\ln(1/\delta)
		\Bigr).
	\]
	Upon combining~\eqref{eq:AlonBenDavidBound} with~\eqref{eq:FromSampleSize_to_FailureProbability}, and observing that
	\[
		\mathcal{R}_{\mathbb{P}}(f|\hat{f})=\mathbb{E}_{X\sim \mathbb{P}_{\operatorname{smpl}}}(h_{f,\hat{f}}(X))
		\mbox{ and }
		\hat{\mathcal{R}}_{\mathbb{P}}^N(f|\hat{f})=\frac1{N}\sum_{n=1}^N(h_{f,\hat{f}})(X_n)
	\]
	we obtain our conclusion upon relabelling $c\eqdef 1/c_1$.
\end{proof}

\section{Toy Numerical Sanity Checks}
\label{s:Main_Results__ss:SanityChecks}
Since the original KAN paper~\cite{KANS_OG_2025} and its many variants~\cite{wu2025graph,zhang2025physics,rong2025recurrent} are now well-established as effective learners, we only perform brief sanity checks to validate our theory and confirm that the addition of the residual connection has a consistently positive impact on the KAN architecture. We first then verify that one can indeed learn a function in the Besov space $B^{\alpha}_{2,2}([0,1])$, but not in $B^{\alpha+1}_{2,2}([0,1])$, as well as its derivatives.
We then verified that the traditional KAN build and our mild variant with residual connections both offer similar performance.
We verify that the training loss converges relatively steadily during training.

\paragraph{A Non-Smooth Besov Function of a Specific Regularity}
The Besov space $B^{\alpha}_{2,2}([0,1])$ coincides with the Sobolev space $H^{\alpha}([0,1])$ for \textit{integer} $\alpha>0$. Thus, we may briefly validate our theoretical results for the family of functions $\{f_{\alpha}\}_{\alpha=1}^{10}$, where for each such $\alpha$ we set
\begin{equation}
	\label{eq:f_alpha}
	f_{\alpha}
	\eqdef
	\frac{x^{\alpha}\log(x)}{\alpha!}
	.
\end{equation}
We chose this example since, $f_{\alpha}$ is a classical example of a function belonging to $H^{\alpha}([0,1])$ but not to $H^{\alpha+1}([0,1])$. To see this, notice that $f_{\alpha}$ is $\alpha$-times differentiable, not only weakly, with $s^{th}$ derivative bounded below near $0$ by
\[
	\frac{d^{{\alpha}+1}}{dx^{{\alpha}+1}} x^{\alpha} \log(x) \succsim \frac{1}{x}
\]
therefore $\lim\limits_{x\downarrow 0} \frac{d^{{\alpha}+1}}{dx^{{\alpha}+1}} x^{\alpha} \log(x) =\infty$. We have chosen the normalization factor of $1/{\alpha!}$ since it is the leading coefficient of the ${\alpha-1}^{st}$ derivative of $x^{\alpha}\log(x)$. Thus, the Sobolev/Besov $\|\cdot\|_{B^{\alpha}_{2,2}([0,1])}$ norm of each $f_{\alpha}$ remains roughly on the same scale (about unity) and are indeed comparing apples-to-apples between derivative levels of $\alpha$.

\paragraph{We Do Actually Learn A Function And Its Derivatives?}
Next, we verify that Theorems~\ref{thrm:Main_Approximation} and~\ref{thm:Main_Generalization} are indeed reflected in practice; namely, that we can both approximate a function and its higher (weak) derivatives from training data. By the Poincaré \'{e}-inequality, it suffices to compare the mean squared error (MSE) between the $s^{th}$ derivative of our prediction and the target data on $[0,1]$. Iterating the Poincaré-inequality implies that we have control over the function itself and all the first $s-1$ derivatives upon controlling our MSE to the $s^{th}$ derivative; \ie
\[
	\| f \|_{L^{2} ([0,1])}
	\lesssim
	\dots
	\lesssim
	\Big\| \frac{\partial^s f}{\partial x^s} \Big\|_{L^{2} ([0,1])}
	\lesssim
	\Big\|
	\frac{\partial^{s+1} f}{\partial x^{s+1}}
	\Big\|_{L^{2} ([0,1])}
	.
\]
Note that, in practice, all derivatives are computed numerically using autograd. Thus, we train on the MSE loss augmented with the MSE between the ${\alpha}^{th}$ derivative of our Res-KAN and the target function, \ie, for a hyperparameter $0\le \lambda\le 1$ controlling our focus on higher derivatives during training, we optimize the following loss function
$$
	\operatorname{loss}_{\lambda}(\hat{f})
	\eqdef
	\frac{(1-\lambda)}{N}\sum_{n=1}^N \big(f(X_n)-\hat{f}(X_n)\big)^2
	+
	\frac{\lambda}{N}\sum_{n=1}^N \Big(\frac{\partial^s f}{\partial x^s}(X_n)-\frac{\partial^s \hat{f}}{\partial x^s}(X_n)\Big)^2
	.
$$
We allow $\lambda$ to linearly decrease from 1 to 0 during training, so that the KAN initially learns the overall function structure (encoded in its derivatives) before fine-tuning its pointwise value approximation in the final SGD iterations.
\begin{figure}[h!]
	\centering
	\begin{subfigure}[b]{0.95\textwidth}
		\includegraphics[width=\textwidth]{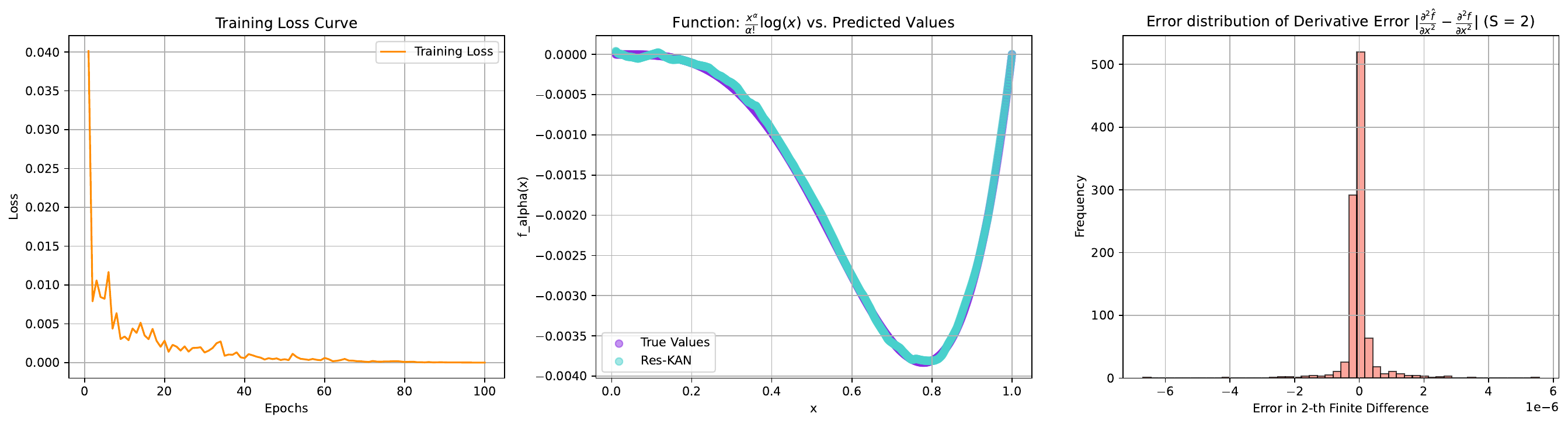}
	\end{subfigure}
	\hfill\\
	\begin{subfigure}[b]{0.95\textwidth}
		\includegraphics[width=\textwidth]{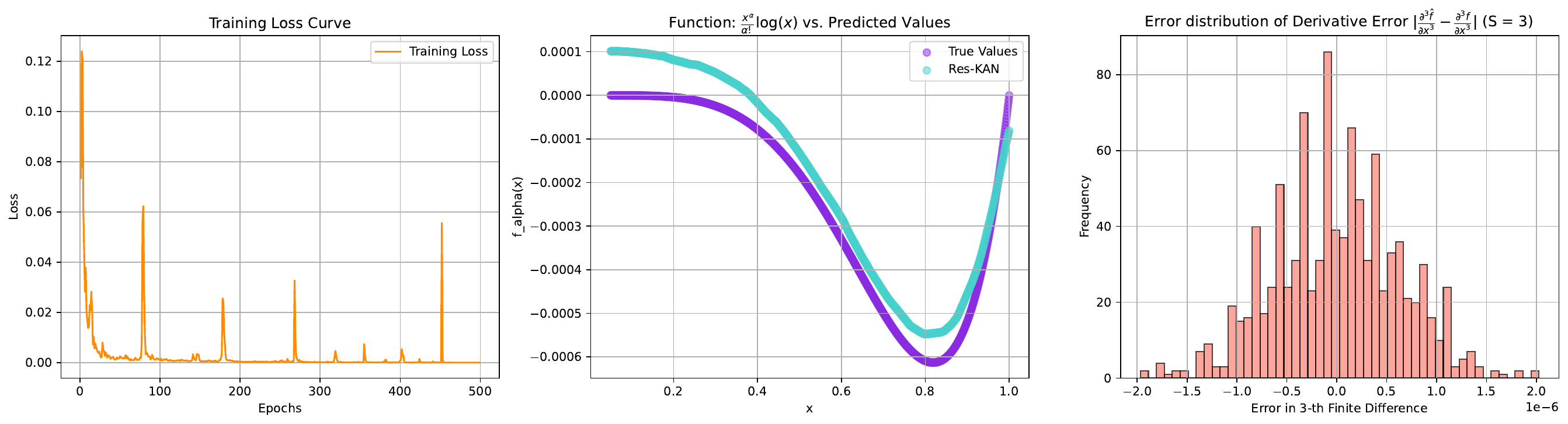}
	\end{subfigure}
	\caption{Evolution of training loss, test prediction, and derivative approximation quality when learning the functions $f_{\alpha}$ in~\eqref{eq:f_alpha} for $\alpha=3$ and $\alpha=5$.}
	\label{fig:checking_fa}
\end{figure}

\paragraph{Do The Residual Connection Adversely Impact the KAN Build?}
We now compare the KAN build with and without the additional residual connection. Our objective is to perform a quick check validating standard profound learning wisdom, that its incorporation can only benefit the model. Each model has width $200$, two hidden layers, and is trained in-sample with $10^3$ datapoints uniformly sampled on $[-2,2]$ then tested, mimicking the realizable setting of Theorem~\ref{thm:Main_Generalization}, on a grid of $10^2$ test points in $[-2,2]$. The experiment confirms that the KAN builds can capture the low-regularity structures, \eg, the spike and rapid a-periodic oscillations, and both builds offer similar performance.
The results of our experiments are plotted as a function of the integer values of $\alpha$ from $1$ to $S=30$ in Figure~\ref{fig:Sobolev__Stability}. As we see, both KAN builds with and without a residual connection offer similar predictive performance, regardless of the degree of Sobolev-Besov regularity available in the target function.

\begin{figure}[h!]
	\centering
	\includegraphics[width=0.55\linewidth]{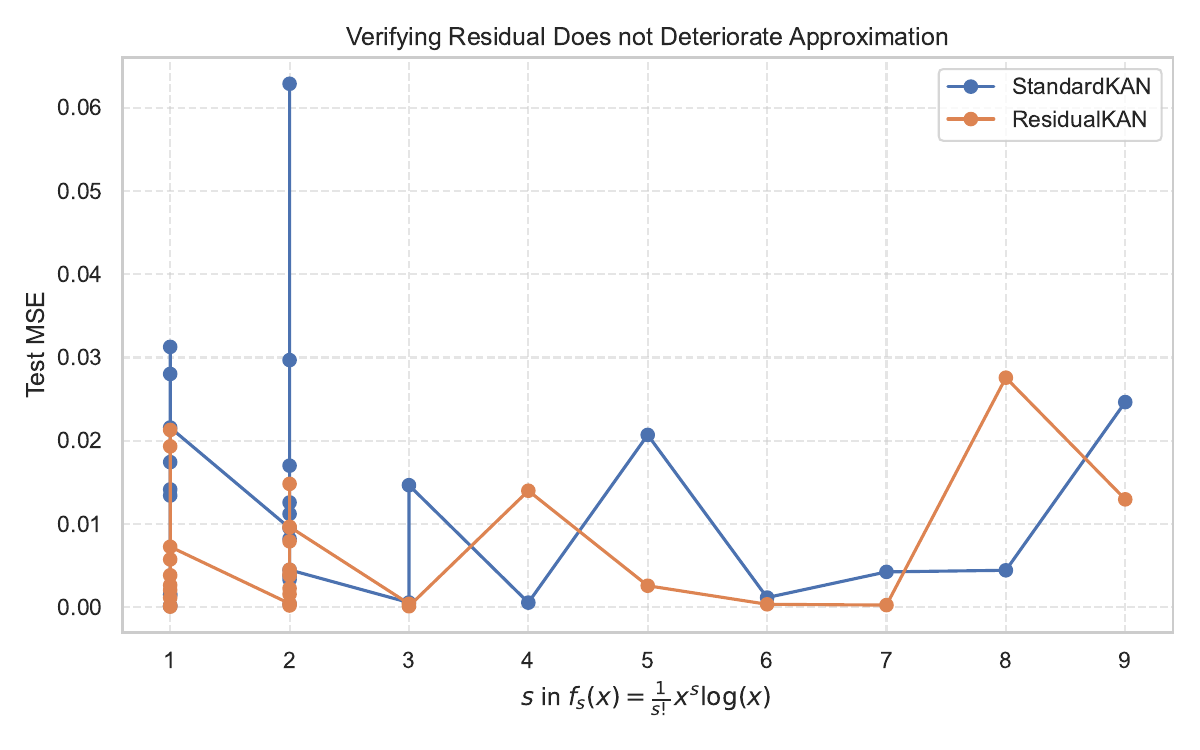}
	\caption{Both the KAN and Res-KAN build offer similar approximation efficacy across different Sobolev-Besov smoothness levels.}
	\label{fig:Sobolev__Stability}
\end{figure}

\paragraph{Functions Beyond our Guarantees}
Lastly, we briefly verify that the KAN with residual can handle pathological regression tasks. Figure~\ref{fig:three_figs} briefly considers two such cases when the target function exhibits a rapid a-periodic oscillatory, \eg, $1/\cos(x)$, or when it exhibits sharp cusps, \eg, $1/\sqrt[10]{x}$. As we see, the training loss descends, the challenging pattern can be learned to reasonable accuracy, and the higher derivatives of the Res-KAN indeed also converge to those of the pathological target function.

\begin{figure}[H]
	\centering
	\begin{subfigure}[b]{0.95\textwidth}
		\includegraphics[width=\textwidth]{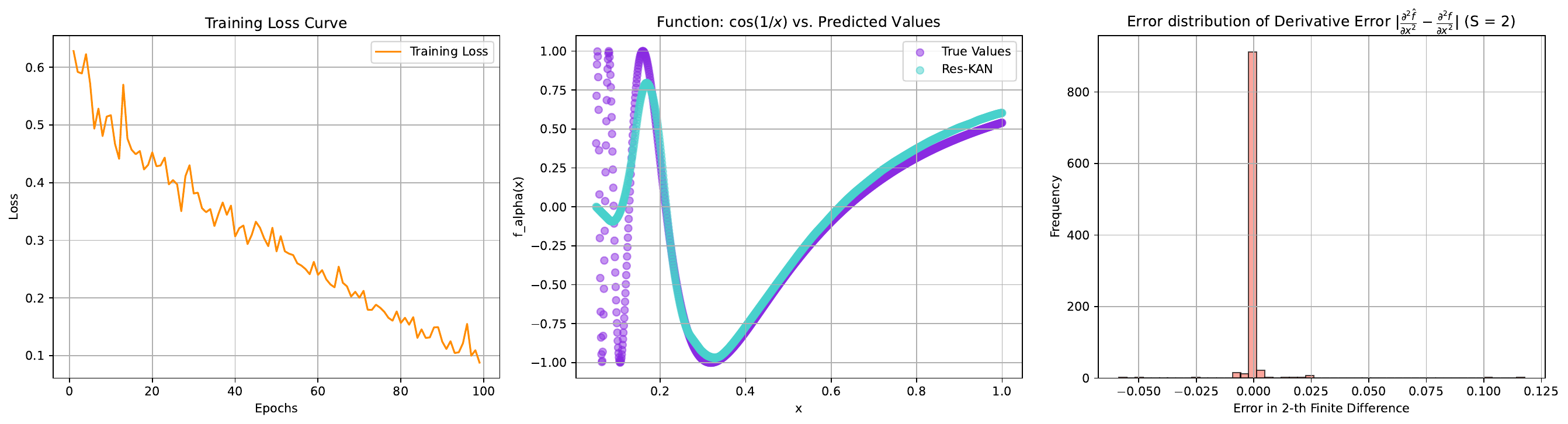}
	\end{subfigure}
	\hfill
	\begin{subfigure}[b]{0.95\textwidth}
		\includegraphics[width=\textwidth]{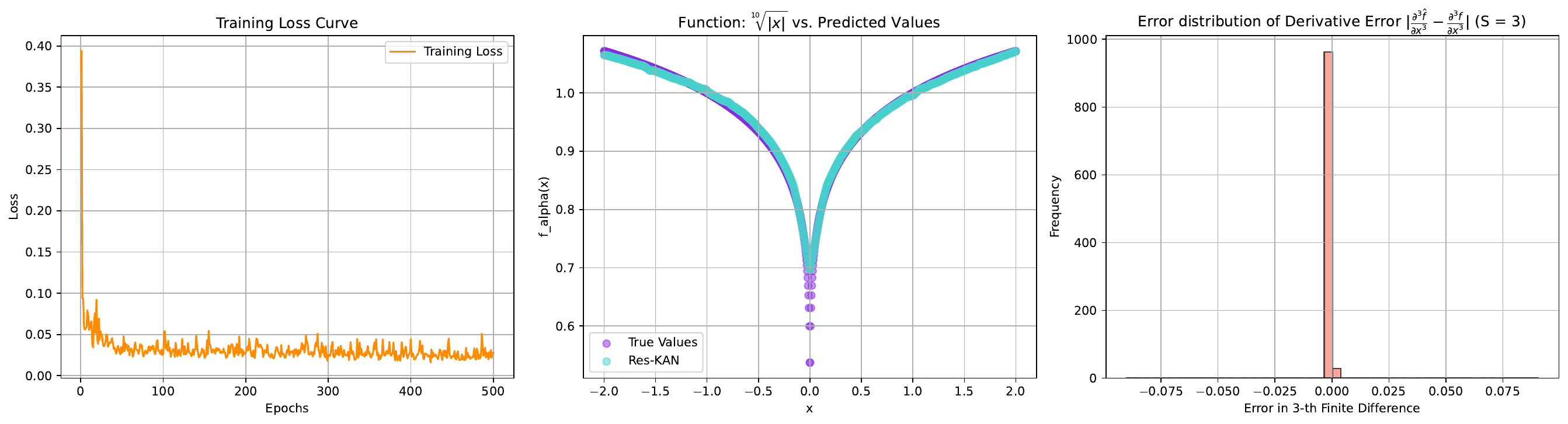}
	\end{subfigure}
	\caption{Two challenging one-dimensional functions to learn - using $\operatorname{ReLU}$ and $\operatorname{ReLU}^2$ MLPs, as well as the standard and the residual KAN builds.}
	\label{fig:three_figs}
\end{figure}

A detailed numerical investigation confirms that, for the relevant class of Besov-type functions and beyond, including a residual connection in the KAN architecture does not negatively impact performance. Consequently, practitioners who prefer implementing the standard, non-residual KAN can expect our results to transfer seamlessly to this simpler setting. Nevertheless, the residual connection may offer practical benefits, particularly in more complex regimes, by enhancing optimization stability and facilitating convergence in deeper or wider network configurations.

\bibliographystyle{acm}
\bibliography{Bookkeaping/Refs}

\end{document}